\documentclass{article}

\usepackage{microtype}
\usepackage{graphicx}
\usepackage{subfigure}
\usepackage{booktabs} %

\usepackage{hyperref}

\usepackage[accepted]{icml2024}

\usepackage{amsmath}
\usepackage{amssymb}
\usepackage{mathtools}
\usepackage{amsthm}
\usepackage{wrapfig}
\usepackage{thmtools}

\usepackage{tablefootnote}
\usepackage[capitalize,noabbrev]{cleveref}

\theoremstyle{plain}
\newtheorem{theorem}{Theorem}[section]

\newtheorem{lemma}[theorem]{Lemma}
\newtheorem{corollary}[theorem]{Corollary}
\theoremstyle{definition}
\newtheorem{definition}[theorem]{Definition}

\theoremstyle{remark}

\usepackage{booktabs,siunitx}
\usepackage{multirow}
\usepackage{makecell}

\usepackage{graphicx}
\usepackage{graphics}
\usepackage{booktabs}
\usepackage{pdflscape}
\usepackage{array}
\newcolumntype{H}{>{\setbox0=\hbox\bgroup}c<{\egroup}@{}} %

\usepackage{amsmath,amsfonts,bm}

\newcommand{\matr}[1]{\bm{#1}}

\newcommand{\set}[1]{\mathcal{#1}}

\mathchardef\mhyphen="2D

\counterwithin*{theorem}{section}

\theoremstyle{definition}
\newtheorem{example}{Example}[section]

\def\eqref#1{equation~\ref{#1}}

\def\1{\bm{1}}

\DeclareMathOperator{\hash}{HASH}

\def\fragment{{f}}
\def\fragmentation{{F}}
\def\vertices{{\set V}}
\def\edges{{\set E}}

\def\features{{\matr X}}
\def\graph{{G}}
\def\vocabulary{{\set Y}}
\def\fragmentationScheme{{\mathcal F}}
\def\WL{{\text{WL}}}
\def\FR{{\text{FR}}}
\def\HLG{{\text{HLG}}}
\def\NF{{\text{NF}}}
\def\ER{{\text{ER}}}
\def\GR{{\text{GR}}}
\def\neighborhood{{\set{N}}}
\def\cond{{\,|\,}}
\def\ours{{FragNet}}
\def\ourslong{{Fragment Graph Neural Network}}
\def\MLP{{\text{MLP}}}
\def\AGG{{\text{AGG}}}
\newcommand*{\ldbrace}{\{\mskip-5mu\{}
\newcommand*{\rdbrace}{\}\mskip-5mu\}}

\DeclareMathAlphabet{\mathsfit}{\encodingdefault}{\sfdefault}{m}{sl}
\SetMathAlphabet{\mathsfit}{bold}{\encodingdefault}{\sfdefault}{bx}{n}

\DeclarePairedDelimiterX{\infdivx}[2]{[}{]}{%
  #1\;\delimsize\|\;#2%
}

\usepackage{enumitem}

\usepackage[textsize=tiny]{todonotes}

\icmltitlerunning{Expressivity and Generalization: Fragment-Biases for Molecular GNNs}

\begin{document}

\twocolumn[
\icmltitle{Expressivity and Generalization: Fragment-Biases for Molecular GNNs}

\icmlsetsymbol{equal}{*}

\begin{icmlauthorlist}
\icmlauthor{Tom Wollschläger}{equal,tum,mdsi}
\icmlauthor{Niklas Kemper}{equal,tum}
\icmlauthor{Leon Hetzel}{tum,helm}
\icmlauthor{Johanna Sommer}{tum,mdsi}
\icmlauthor{Stephan Günnemann}{tum,mdsi}
\end{icmlauthorlist}

\icmlaffiliation{tum}{School of Computation, Information \& Technology, Technical University of Munich.}
\icmlaffiliation{mdsi}{Munich Data Science Institute, Germany.}
\icmlaffiliation{helm}{Helmholtz Center for Computational Health, Munich, Germany}
\icmlcorrespondingauthor{Tom Wollschläger}{t.wollschlaeger@tum.de}
\icmlcorrespondingauthor{Niklas Kemper}{niklas.kemper@tum.de}

\icmlkeywords{Machine Learning, ICML}

\vskip 0.3in
]

\printAffiliationsAndNotice{\icmlEqualContribution} %

\begin{abstract}

Although recent advances in higher-order Graph Neural Networks (GNNs) improve the theoretical expressiveness and molecular property predictive performance, they often fall short of the empirical performance of models that explicitly use fragment information as inductive bias. However, for these approaches, there exists no theoretic expressivity study. In this work, we propose the \textit{Fragment-WL} test, an extension to the well-known Weisfeiler \& Leman (WL) test, which enables the theoretic analysis of these fragment-biased GNNs. Building on the insights gained from the Fragment-WL test, we develop a new GNN architecture and a fragmentation with infinite vocabulary that significantly boosts expressiveness. We show the effectiveness of our model on synthetic and real-world data where we outperform all GNNs on Peptides and have $12\%$ lower error than all GNNs on ZINC and $34\%$ lower error than other fragment-biased models. Furthermore, we show that our model exhibits superior generalization capabilities compared to the latest transformer-based architectures, positioning it as a robust solution for a range of molecular modeling tasks.

\end{abstract}

\section{Introduction}
\label{sec:intro}
A common issue with Graph Neural Networks (GNNs) is their lack of expressiveness, including their inability to recognize substructures, which could limit their empirical performance \cite{morris_weisfeiler_2021}. In chemistry and machine learning for chemistry, frequently occurring substructures, or fragments, are commonly used to improve expressivity, as they can be powerful predictors of the functional properties of molecules \cite{MERLOT2003594}. 
The beneficial effect of fragments becomes even more apparent in larger systems, such as proteins, where individual components often resemble whole residues \cite{singh_2003}.

To give GNNs the predictive power of substructures, recent work has introduced more powerful methods that relate expressive power to the ability to distinguish isomorphic graphs. For these higher-order GNNs, the Weisfeiler \& Leman (WL) test serves as a measure of expressivity. Empirical evaluations analyze the ability of these models to count substructures and predict molecular properties \cite{zhang2023expressive, morris_weisfeiler_2021}. However, many of these works focus primarily on theoretical expressiveness analysis, often neglecting practical implications. 

Recently, \citet{campi_expressivity_2023} showed that these models suffer from poor generalization to data that do not perfectly fit the training distribution. This often results in sub-par performance on real-world data \cite{maron2019provably, gasteiger2020directional}. At the same time, models that explicitly incorporate fragment information as inductive biases perform better overall \cite{fey_hierarchical_2020}, but are often limited to a single substructure and lack a theoretical analysis of their expressiveness, leading to a mismatch between theory and practical performance \cite{zhu_mathcalo-gnn_2022, fey_hierarchical_2020, zang2023hierarchical}.

In this work, we address the gap between theory and practical performance by introducing the Fragment-WL test, an extension of the standard WL test that enables a unified analysis of existing (fragment-biased) models. In addition, we introduce a new powerful model that directly leverages graph fragments within its message-passing framework. Thereby, our model becomes more robust to varying graph structures and can generalize better to out-of-distribution data. Lastly, our model enables a new fragmentation to represent molecular graphs with an infinite vocabulary consisting only of basic building blocks. We show the usefulness of our approach across a range of molecular datasets and tasks. We analyze both short- and long-range interactions, where we achieve state-of-the-art performance compared to graph-based models and even outperform transformer-based architectures in some scenarios.\footnote{Find our code at \href{https://www.cs.cit.tum.de/daml/fragment-biased-gnns/}{cs.cit.tum.de/daml/fragment-biased-gnns/}}

Our core contributions can be summarized as follows: 
\begin{itemize}
    \item We provide a more fine-grained hierarchy on the expressiveness for a multitude of models that incorporate substructures as inductive bias, such as including them as node features, learning an individual representation, or performing operations on higher-level structures.
    \item We propose a new architecture that performs message passing along substructures and improves expressivity and generalization while retaining linear complexity. 
    \item With our new architecture, we can propose a novel fragmentation for molecules that handles an infinite number of substructures based on simple yet flexible building blocks but still generalizes well.
    \item We study predictive power, long-range performance, and generalization through extensive experiments.
\end{itemize}

\section{Background}
\textbf{Notation.}
A graph $\smash{\graph \coloneqq (\vertices,\edges,\features)}$ consists of a set of vertices $\smash{\vertices}$, a set of (undirected) edges $\smash{\edges \subseteq \vertices \times \vertices}$ and $d$ node features $\smash{\features\in \mathbb{R}^{|\vertices| \times d}}$ for every node $\smash{v \in \vertices}$. The set of nodes that are adjacent to $v$ is denoted by $\smash{\neighborhood(v)}$. Two graphs $\smash{\graph^1 = (\vertices^1, \edges^1, \features^1)}$ and $\smash{\graph^2 = (\vertices^2, \edges^2, \features^2)}$ are \textit{isomorphic} if there exists a bijection $\smash{b: \vertices^1 \to \vertices^2}$ that preserves edges and node features, that is, $\smash{\{v,w\} \in \edges^1 \Leftrightarrow \{b(v), b(w) \} \in \edges^2}$ and $\smash{\features^1_v = \features^2_w}$. For a subset of nodes $\smash{\set U \subseteq \vertices}$, we denote the induced subgraph with respect to these nodes by $\smash{\graph [\set U]}$.

\textbf{Expressiveness.}
We can classify the expressiveness of functions over graphs by their capability to distinguish non-isomorphic graphs. We say that a function $f$ is (in parts) \emph{more powerful} than a function $g$ if there exist two non-isomorphic graphs $\graph^1$, $\smash{\graph^2}$ such that $\smash{f(\graph^1) \neq f(\graph^2)}$ whereas $\smash{g(\graph^1) = g(\graph ^2)}$. The function $f$ is \emph{strictly more powerful} than $g$ (we write $f > g$) if $f$ is more powerful than $g$ and $g$ is not (in parts) more powerful than $f$.

\textbf{Weisfeiler \& Leman.}
The Weisfeiler \& Leman graph isomorphism test is an iterative graph coloring algorithm that bounds the expressive power of MPNNs \cite{kiefer2022power}. In each iteration, it produces a color for each node based on its neighboring nodes' colors. Starting with a vertex color based only on features $\smash{c_v^0 = \hash(\features_v)}$, we calculate the update for the color $c$ of node $v$ in iteration $t$:
\begin{equation}
    c_v^{(t)} = \hash\big(c_v^{(t-1)}, \ldbrace c_w^{(t-1)} \cond w \in \neighborhood(v) \rdbrace\big).
    \label{eq:hash}
\end{equation}
The algorithm terminates once the set of unique colors does not increase. Two non-isomorphic graphs can be distinguished if the multiset of colors differs at the end. As this test cannot distinguish all non-isomorphic graphs, it can be extended to strictly more powerful versions, $k$-WL, incorporating $k$-tuples of nodes to determine the color. For more background information, we refer to \citet{morris_weisfeiler_2021}. Importantly, $2$-WL is equivalent to the previously described WL test \cite{huang2021tutorial}. 

\textbf{Fragmentations.}
A vocabulary $\vocabulary$ is a set of graphs (potentially including node features) representing important substructures, e.g., cycles. A fragment of a graph $\graph$ is an induced subgraph $\graph[\fragment]$ isomorphic to a graph from the vocabulary. We will identify a fragment simply by the subset of nodes $\fragment \subseteq \vertices$. All fragments $f$ that are isomorphic to the same graph of the vocabulary have the same $\text{type}(f)$. A fragmentation scheme $\fragmentationScheme$ is a permutation invariant function that maps a graph $\graph$ to a set of fragments $\fragmentationScheme(G) =: \fragmentation$, which is called a fragmentation. Note that there might exist subgraphs isomorphic to a graph in $\vocabulary$ that are \emph{not} in $\fragmentationScheme(G)$. For example, even if $\vocabulary$ contains 5-cycles, not all 5-cycles in $\graph$ need to be in $\fragmentationScheme(G)$. If, for all graphs $G$, $\fragmentationScheme(\graph)$  includes \emph{every} subgraph isomorphic to a graph $V \in \vocabulary$, we say that the fragmentation scheme $\fragmentationScheme$ recovers $V$.

\section{Related work}
\label{sec:related_work}
\textbf{Expressiveness of GNNs.}
Message-Passing Neural Networks (MPNN)\footnote{We use MPNNs and GNNs interchangeably.} are limited in their expressiveness. Their ability to distinguish between non-isomorphic graphs is confined to the 2-WL algorithm, restricting their discriminative power \cite{xu_how_2019}. Moreover, when it comes to recognizing substructures, MPNNs are unable to accurately count almost all types of substructures \cite{chen_can_2020}. This limitation stems from their reliance on purely local messages, which---despite facilitating excellent linear space and time complexity---renders them blind to higher-level structural information within graphs.

\textbf{Higher-order GNNs.}
In response to the inability to effectively learn substructures, the introduction of more powerful GNN architectures aims to overcome this limitation and enable comprehensive substructure learning.
\citet{morris_weisfeiler_2021} draw inspiration from the multidimensional $k$-WL algorithm and diverge from learning node-specific representations by considering each $k$-tuple of nodes instead. Although this improves expressiveness, its complexity increases exponentially.
Subgraph GNNs comprise an alternative to improve substructure identification, decomposing a graph into smaller subgraphs for GNN application. The resulting subgraph representations are pooled before a final graph level representation is derived \citep{huang_boosting_2023, frasca_understanding_2022}. With some strategies for extracting subgraphs, subgraph GNNs can identify basic substructures such as 4-cycles \cite{huang_boosting_2023}. \citet{puny2023equivariant} extend the WL test for higher-order GNNs to the graph polynomial counting problem as a new expressivity measure, highlighting the importance of more fine-grained tests for GNNs. However, the limitations of higher-order GNNs lie in their inability to effectively learn more intricate substructures, accompanied by an increase in time complexity. Recent findings also suggest susceptibility to adversarial attacks and out-of-distribution data, hinting at challenges in robustly learning substructures \cite{campi_expressivity_2023}.

\textbf{Fragment-Biased GNNs.}
Another line of work provides fragment information to GNNs as an explicit inductive bias. These fragment-biased models vary not only in their vocabulary but also in the way fragmentation information is integrated into the model. \textit{Node features}: \citet{bouritsas_improving_2023} introduce GSN-v, which uses the number of cycles or cliques as an additional node feature. \textit{Learned fragment representation}: Instead of treating fragmentation information as a fixed feature, other models learn representations for each fragment by aggregating information from the corresponding nodes. \citet{zhu_mathcalo-gnn_2022} use a vocabulary of only cycles whereas \citet{zang2023hierarchical} present HiMol, which fragments a molecular graph, based on chemical properties. \textit{Higher-level graph}: A natural extension of the learned fragment representation is a higher-level graph of fragments where neighboring fragments influence each other. \citet{Thiede2021AutobahnAG} use equivariant computations along the paths of length 3 to 6 and cycles of sizes 5 and 6. \citet{fey_hierarchical_2020} build a higher-level junction tree using rings and edges. 
Yet, none of the existing works compare---theoretically or experimentally---how to encode and use substructure information in the model. Additionally, most works only focus on a single substructure that does not allow to fragment the complete graph.

\textbf{Topological GNNs} use higher-level topological structures such as simplicial complexes \cite{bodnar_weisfeiler_2021} or CW-Networks \cite{bodnar_weisfeiler_2022, giusti_cin_2023} in their message-passing schemes. While coming from a different theoretical direction than substructure-biased GNNs, in practice, they use cliques or cycles as learned fragment representations or in a higher-level graph.

\textbf{Graph Transformer.} Recently, models such as Graph Transformers \cite{ying_transformers_2021, ma_graph_2023, geisler2023transformers} and ViT/MLP-Mixers \cite{he_generalization_2023} for graphs adapted successful models from other domains to graph data.
Their ability to recognize substructures depends on the positional encoding used. Almost all recent models use random walk encodings, which can help to discover simple substructures like cycles. 

\textbf{Fragmentation Schemes.}
Fragmentation methods in the chemical domain aim to divide a molecular graph into subgraphs with distinct structures or properties. There are various strategies to achieve this, such as separating probable reactants \cite{brics_decomp}, categorizing molecules into distinct structural classes \cite{bemismurcko}, or breaking apart acyclic bonds \cite{maziarz_learning_2022, jin_junction_2019}. Unlike these methods, data-driven approaches like those outlined in \citet{kong_molecule_2022} and \citet{geng2023novo} focus on deriving subgraphs directly from a dataset without relying on predefined rules for decomposition.

\section{Weisfeiler \& Leman Go Fragments}
\label{sec:theory}
Existing fragment-biased MPNNs vary in their underlying fragmentation scheme and how the fragment information is incorporated into the model. This variability makes a direct comparison of the expressiveness of these models difficult. To address this challenge, we propose a new, more fine-grained version of the WL test, called \emph{Fragment-WL} test, that incorporates detailed structural elements. We derive a hierarchy of tests that capture how fragmentation information is incorporated into existing substructure-biased models, while leaving out all variability that does not influence the expressivity. 

Our Fragment-WL test also subsumes existing WL variants designed for simplicial complexes and CW-cells \cite{bodnar_weisfeiler_2021, bodnar_weisfeiler_2022}, providing a more unified framework for assessing the expressiveness of both substructure-biased and topological GNNs. Furthermore, our proposed Fragment-WL test highlights the significance of how fragment information is incorporated into the model, emphasizing that the integration methodology plays a crucial role for determining the model's expressive power.

Fragment-WL entails multiple variants with increasing expressiveness in distinguishing isomorphic graphs. In the following, we first provide a general framework and then define the individual Fragment-WL versions that perform the original WL test on different augmented graphs. We start with a definition of WL tests on augmented graphs:
\begin{definition}
    A $g$-WL test is a function that performs the WL test on the augmented graph $g(\graph)$, i.e.
    \[g\mhyphen\WL(\graph) := \WL(g(\graph))\]
    where $g$ is a function mapping from graphs to graphs, i.e., $g: (\vertices, \edges, \features) \mapsto (\vertices', \edges', \features')$.
\end{definition}

\begin{figure}[t]
\centering
\includegraphics[width=\linewidth]{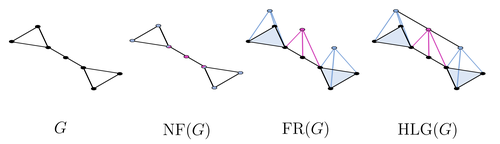}
\vskip -0.05in
\caption{Example graph $\graph$ with corresponding augmented variants. $\NF(\graph)$ includes node features, $\FR(\graph)$ also includes a representation for each fragment and $\HLG(\graph)$ also has connections between neighboring fragment represenations.}
\label{fig:fragment-wl}
\end{figure}

There are three ways in which a fragmentation $\fragmentation$ is used in existing fragment-biased GNNs: as an additional node feature, as learned fragment representation, and as a higher-level graph. We instantiate $g$ with the corresponding functions to augment the graph with the respective features. 
First, we use additional node features. We extend the individual node features with the information of the fragments that the node is contained in. We concatenate this information to the already existing features. Formally, we define this augmentation function in the following way.
\\
\begin{definition}
    \label{def:NF}
    We define the node feature function  as $\NF(\vertices, \edges, \features)=(\vertices, \edges, \features^\NF)$ with: 
    \begin{align*}
    &\features^\NF_v = X_v \: \lVert \: \lambda\bigl(\ldbrace\text{type}(\fragment) \cond v \in \fragment, \fragment \in \fragmentation\rdbrace\bigr),   
    \end{align*}
    where $\lambda$ represents any injective function and $\lVert$ indicates the concatenation operation. We instantiate $g$ with $\NF$ to create the \NF-WL test.  
\end{definition}

Another prominent way is using representations for each fragment and messages that are flowing from the lower level nodes to their entailing fragment and backwards. This means that we introduce a new vertex for each fragment and connect it to all its corresponding vertices in the original graph. We depict this graph $\FR(\graph)$ in \cref{fig:fragment-wl}. We define this augmentation function in the following. 
\begin{definition}
    We define the fragment representation function as $\FR(\vertices, \edges, \features) = (\vertices^\FR, \edges^\FR, \features^\FR)$ with:
    \begin{align*}
    &\vertices^\FR := \vertices \cup \fragmentation, \\
    &\edges^\FR := \edges \cup \bigl\{ \{\fragment,v\} \: | \: \forall \fragment \in \fragmentation, \forall v \in f \bigr\},\\
    &\features^\FR_i := \begin{cases} \features_i & i \in \vertices \\
    \text{type}(i) & i \in \fragmentation 
    \end{cases}
    \end{align*}
\end{definition}

Lastly, we allow messages to be exchanged between neighboring fragments, thus creating a higher-level graph on which information can flow. To this end, we add edges between fragments that have neighboring nodes in $G$ and thus construct a graph of the higher-level fragments, see $\HLG(\graph)$ in \cref{fig:fragment-wl}.
\begin{definition}
The higher-level graph augmentation functions is $\HLG(\vertices, \edges, \features) = (\vertices^\HLG, \edges^\HLG, \features^\HLG)$ with:
\begin{align*}
&\vertices^\HLG := \vertices^\FR, \: \: \:  \: \features^\HLG := \features^\FR,\\
&\edges^\HLG := \edges^\FR \cup \big \{ \{\fragment,k\} \mid \fragment,k \in \fragmentation, \fragment \cap k \neq \emptyset \big  \}
\end{align*}

\end{definition}

Equipped with the formal definitions, we will now compare the power of performing the WL test on these transformed graphs to the original graph\footnote{All proofs are detailed in \cref{app:proofs}}. The power of the Fragment-WL test depends on the fragmentation scheme $\mathcal{F}$. With a sufficiently advanced fragmentation scheme, even NF-WL can become arbitrarily powerful. 
\begin{restatable}{theorem}{theopowerful}
\label{theo:powerful}
There exist fragmentation schemes such that NF-WL, FR-WL and HLG-WL are all strictly more powerful than $k$-WL for any $k$.
\end{restatable}
However, in practice, mostly fragmentation schemes with a vocabulary of rings, paths, and cliques are used for fragment-biased GNNs \cite{fey_hierarchical_2020, bouritsas_improving_2023, zhu_mathcalo-gnn_2022}. So, from now on, we will restrict ourselves to such fragmentation in our theoretical analysis. Next, we show that it matters how to incorporate fragment information and that the WL variants become more powerful through higher-level abstraction.

Integrating fragment information from any non-trivial substructure as an additional node feature already increases expressiveness beyond 2-WL. 
\begin{restatable}{theorem}{theoNF}
    \label{theo:NF}
    NF-WL is strictly more powerful than 2-WL for fragmentation schemes $\mathcal{F}$ that recover any substructure with more than two nodes.
\end{restatable}
This shows that the classical 2-WL test cannot reveal differences in the expressivity of fragment-biased GNNs since using any substructure as a node feature already increases expressivity beyond 2-WL. Our Fragment-WL test, however, provides a more fine-grained alternative that reveals that higher-level abstraction through a learned fragment representation strictly increases the expressivity compared to node features:
\begin{restatable}{theorem}{theoFR}
    \label{theo:FR}
    FR-WL is strictly more powerful than NF-WL for fragmentation schemes $\mathcal{F}$ recovering 3-cycles.
\end{restatable}
Building a higher-level graph of fragments further increases the expressivity. Figure \ref{fig:HLG} shows an example of two graphs that are indistinguishable by 2-WL, NF-WL, and FR-WL but distinguishable by HLG-WL. Formally, we express this in the following theorem.
\begin{figure}[t]
\centering
\includegraphics[width=\linewidth]{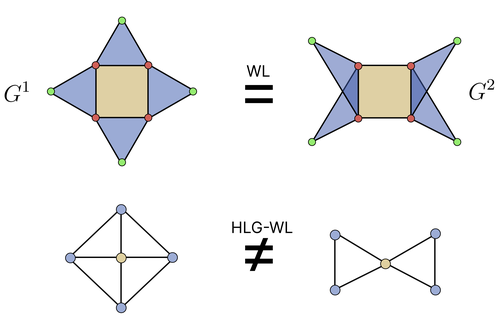}
\vskip -0.1in
\caption{Graphs $\graph^1$ and $\graph^2$ with their corresponding higher-level graph of fragments. The edges of the fragment representation to the vertices of $\graph^1$ and $\graph^2$ are omitted. $\graph^1$ and $\graph^2$ are indistinguishable by WL, \NF-WL and \FR-WL but distinguishable by HLG-WL as the higher-level graphs exhibit different connections from the 3-ring nodes to the 4-ring node.}
\label{fig:HLG}
\end{figure}
\begin{restatable}{theorem}{theoHLG}
    \label{theo:HLG}
    HLG-WL is strictly more powerful than FR-WL for fragmentation schemes $\mathcal{F}$ recovering 3-cycles.
\end{restatable}

Hence, we have shown that the expressivity increases strictly monotonically from 2-WL to HLG-WL for fragmentation schemes recovering 3-cycles:
\begin{equation}\label{hierarchy}
    \text{2-\WL} < \text{\NF-\WL} < \text{\FR-\WL} < \text{\HLG-\WL}
\end{equation}
Regarding the higher-dimensional $k$-WL hierarchy, our Fragment-WL hierarchy cannot be bounded by 3-WL if the fragmentation can recover 5 cycles. 
\begin{restatable}{theorem}{theoHLGpower}
    \label{theo:HLG power}
    HLG-WL is in parts more powerful than 3-WL for fragmentation schemes $\mathcal{F}$ recovering 5-cycles.
\end{restatable}

Our developed Fragment-WL hierarchy models the different ways in which fragment information is used in most fragment-biased and topological GNNs. This new measure of expressiveness allows the comparison and ordering of these existing methods; see \cref{tab:expressiveness} for an overview.

In summary, our Fragment-WL test provides a new alternative measure of expressivity compared to the original WL test. Our developed hierarchy reveals that it matters how to incorporate fragmentation information, i.e., higher-level abstraction increases expressivity. Additionally, it allows for a comparison of the expressiveness of most existing fragment-biased and topological GNNs.

\section{\ourslong{}}
Based on the insights of the previous section, we propose our new model architecture and a new fragmentation scheme with infinite vocabulary consisting of only basic building blocks. %
Given the higher expressiveness, our model can differentiate complex substructures given only these basic building blocks, as it is able to learn the dependencies on the higher-level graph. We empirically confirm this in \cref{sec:experiments}. 

\subsection{Model}
Building on the theoretical findings that a higher-level graph (see \cref{theo:HLG}) improves expressiveness, we propose \ours{}, a general model for any fragmentation $\fragmentation$ that performs message-passing on the original graph \emph{and} a higher-level graph of fragments that are also connected, i.e., we are using the \HLG{} augmented graph. 

The learned representation $h^t_v$ at step or layer $t$ of a node receives a message from neighboring nodes $m^t_{\vertices\rightarrow v}$ and fragments $m^t_{\fragmentation \rightarrow v}$. Similarly, the learned representation $h^t_f$ of a fragment receives a message from neighboring \linebreak fragments  $m^t_{\fragmentation \rightarrow f}$ and nodes $m^{t}_{\vertices\rightarrow f}$. 
\begin{align*}
m^t_{\vertices\rightarrow v} &= \AGG\bigl(\ldbrace \MLP(h^{t-1}_u, h^{t-1}_e) \mid e = \{u,v\} \in \edges \rdbrace\bigr) \\
m^{t}_{\fragmentation \rightarrow f} &= \AGG\bigl(\ldbrace h^{t-1}_g \mid g \in N_\fragmentation(f) \rdbrace\bigr) \\
m^t_{\fragmentation \rightarrow v} &= \AGG\bigl(\ldbrace h^{t-1}_f \mid v \in f, f \in \fragmentation \rdbrace\bigr) \\
m^{t}_{\vertices\rightarrow f} &= \AGG\bigl(\ldbrace h^{t-1}_v \mid v \in f, v \in \vertices \rdbrace\bigr) 
\end{align*}
where $N_F(f)$ denotes the neighbors of fragment $f$ in the higher-level graph of fragments. The hidden representations are updated by combining the incoming messages with the previous hidden representation:
\begin{align*}
h^t_v &= \MLP(h_v^{t-1}, m^t_{\vertices\rightarrow v}, m^t_{\fragmentation \rightarrow v}) \\
h^t_f &= \MLP(h^{t-1}_f, m^{t}_{\vertices\rightarrow f}, m^t_{\fragmentation \rightarrow f}) \\
h^t_e &= \MLP(h^{t-1}_e, h^{t-1}_u, h^{t-1}_v) 
\end{align*}

\begin{figure}[t]
\centering
\includegraphics[width=\linewidth]{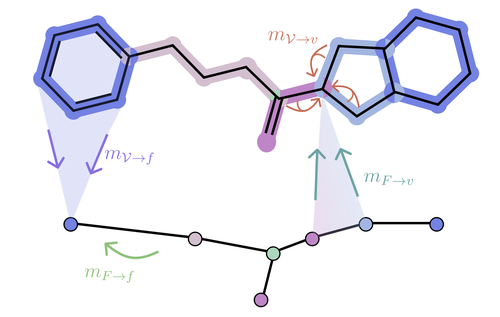}
\vskip -0.05in
\caption{Overview of our model and our fragmentation. The molecular graph is fragmented with our rings-paths fragmentation into three cycles, three paths, and a junction node.
The figure shows the messages $m^{t}_{\fragmentation \rightarrow f}$, $m^{t}_{\vertices\rightarrow f}$  to one fragment $f$, and the messages $m^t_{\vertices\rightarrow v}$, $m^t_{\fragmentation \rightarrow v}$ to one vertex $v$.}
\label{fig:overview}
\end{figure}

The final graph-level readout after $T$ layers is computed by aggregating the multiset of node representation, edge representations, and fragment representations:
\begin{align*}
    \text{OUT}(\ldbrace h^T_v \mid v \in \vertices \rdbrace, \ldbrace h^T_e \mid e \in \edges \rdbrace,  \ldbrace h^T_f \mid f \in \fragmentation \rdbrace)
\end{align*}

Note that the complexity of our \ours{} model is linear in the number of nodes and fragments (assuming that each node is only part of a constant number of fragments).

Additionally, our \ours{} model achieves the highest expressiveness in our Fragment-WL hierarchy and also compared to other fragment-biased GNNs.
\begin{restatable}{theorem}{ourmodel}
    \label{our model}
    \ours{}s are at most as powerful as HLG-WL. Additionally, when using injective neighborhood aggregators and a sufficient number of layers, \ours{}s are as powerful as HLG-WL.
\end{restatable}

\subsection{Molecular fragmentation}
Apart from the question of how to use a fragmentation, there is the equally important question of how to fragment the graph in the first place. A fragmentation $\fragmentation$ has to fulfill two seemingly conflicting goals:
\begin{enumerate}
\itemsep0em
    \item $\fragmentation$ should contain all \emph{important} substructures.
    \item $\fragmentation$ should facilitate generalization.
\end{enumerate}
On the one hand, if the fragmentation is too coarse-grained, important structural features that the model cannot learn may be missed. On the other hand, if the fragmentation is too fine-grained, it is harder to find similarities between graphs, and it exposes the model to the risk of overfitting. Of course, the right level of detail and the right fragmentation scheme critically depend on the application domain. Existing approaches to fragment molecules focus either only on a single substructure \cite{zhu_mathcalo-gnn_2022} or fragment the molecule based on chemical properties \cite{degen_art_2008}, which requires a huge vocabulary.

Our approach is capable of fragmenting the complete graph with only two classes of substructures: rings and paths. At a conceptual level, first, all minimal rings are extracted. Second, the remaining edges are connected at nodes of degree two to form paths.
Third, whenever three or more fragments meet at a node, a junction node is introduced in the higher-level graph and the fragments are only connected to the junction nodes. This prevents most cycles in the higher-level graph. \cref{fig:overview} shows our fragmentation on an example molecule. 

Since our rings-paths fragmentation scheme recovers 5-cycles, we derive the following Corollary from \cref{theo:HLG power}.
\begin{corollary}
    \ours{} with our rings-paths fragmentation scheme is in parts more powerful than 3-WL.
\end{corollary}

While this fragmentation increases the expressiveness and makes it possible to fragment the complete graph, it could require a very large vocabulary, as it has to include paths and cycles of any size found in the dataset. To still effectively encode the fragment type, we introduce a novel ordinal encoding.

\paragraph{Ordinal encoding.} Previous works either use no encoding for the types of fragments \cite{fey_hierarchical_2020} or a simple one-hot encoding \cite{bouritsas_improving_2023}. However, to facilitate the generalization capabilities of a model, the encodings of similar fragments should also be similar.
 In our approach, we introduce an ordinal fragment encoding, which accomplishes this by incorporating two embeddings: one for the fragment class (i.e., $\text{class}(f) \in \{\text{path}, \text{cycle}, \text{junction} \}$) and another that is proportionally scaled based on the fragment size; see \cref{fig:ordinal}. More formally:
 \begin{align*}
  h_f^0 = (e_{\text{class}(f)}, |f| \cdot s_{\text{class}(f)}),
\end{align*}
where $s$ and $e$ are different learned embeddings for the classes of cycle, path, and junction fragments. 
This approach enables the encoding of an infinite vocabulary, accommodating even completely unseen fragments while concurrently supporting effective model generalization.

\begin{figure}[t]
\includegraphics[width=\linewidth]{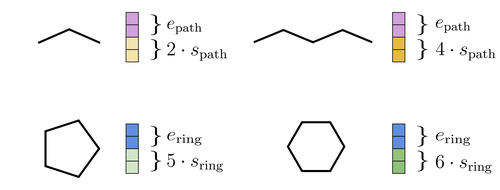}
\centering
\caption{Ordinal encoding applied to a 2-path, 4-path, 5-ring, and 6-ring. The encoding comprises two components: one learned embedding $e$ for every fragment class (i.e., path, cycle, or junction) and another learned embedding $s$ that is proportionally scaled based on the fragment size.}
\label{fig:ordinal}
\end{figure}

\subsection{Long-range interaction}
Many graph-related tasks demand the capability to capture long-range interactions among distant nodes \cite{dwivedi2023long}. 
However, GNNs suffer from a phenomenon called \textit{over-squashing}, where nodes are insensitive to information from distant nodes. This occurs because, in GNNs, the computational graph is equal to the input graph, requiring $L$ layers to exchange information between nodes separated by a distance of $L$. \citet{di_giovanni_over-squashing_2023} show that increasing the number of layers does not help prevent this phenomenon, as GNNs suffer from vanishing gradients. Additionally, their work introduces the commute time between nodes as an indicator for over-squashing. 

Our rings-paths fragmentation, together with the higher-level graph, create shortcuts that can help mitigate over-squashing. Messages can traverse the smaller, higher-level graph of fragments, while the message path still has semantic meaning. 
In \cref{sec:experiments}, we demonstrate that our rings-paths fragmentation reduces the commute time in both synthetic and real-world datasets, enhances the recoverability of messages from distant nodes, and shows strong performance on benchmarks requiring long-range interactions.

\section{Results}
\label{sec:experiments}
While we have theoretically demonstrated that our model attains the highest expressiveness within our Fragment-WL hierarchy, we also empirically evaluate its expressiveness by examining its ability to count substructures. Additionally, we explore its ability to communicate over long distances, its overall predictive effectiveness, and its capacity to generalize.

\textbf{Expressiveness.} To evaluate how well our model can learn to recognize chemically \emph{important} substructures in molecular graphs, we first identify the most common substructures in the ZINC 10k dataset \cite{gómezbombarelli2017automatic} using a chemically-inspired fragmentation scheme, specifically MAGNet \cite{hetzel_magnet_2023}. Subsequently, we train our model to predict substructure counts. Our model is able to identify all substructures close to perfection as demonstrated in \cref{tab:substruct_counting} on a subset of fragments; performance details for all substructures are available in \cref{app:further_experiments}.
  Notably, our model achieves high accuracy even for intricate substructures not present in our vocabulary. This underscores that our fragmentation, based solely on rings and paths, together with our ordinal encoding and the higher-level message passing, proves sufficient for the model to recognize more complex substructures. This is essential for application in, e.g., fragment-based molecule generation, as the task of the encoder is to encode information about such substructures.
\begin{table}[t]
    \centering
    \caption{Our model predicting the number of different substructures occurring in ZINC test dataset.}
    \vskip 0.05in
    \resizebox{\columnwidth}{!}{\begin{tabular}{l|HHHllHllHHlHHHHHHHHHHHHHHHHHH}
\toprule
     \multirow{2}{*}{\textbf{Fragment}}&  
     \multirow{2}{*}{\includegraphics[width=0.05\textwidth]{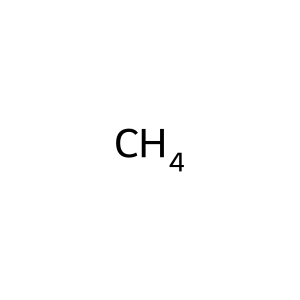}} &  \multirow{2}{*}{\includegraphics[width=0.05\textwidth]{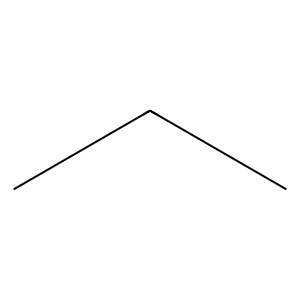}} &  \multirow{2}{*}{\includegraphics[width=0.05\textwidth]{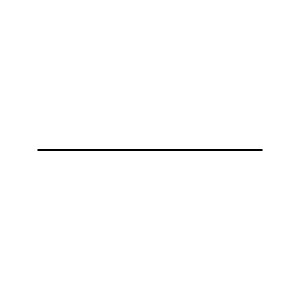}} &  \multirow{2}{*}{\includegraphics[width=0.05\textwidth]{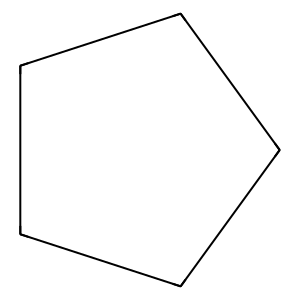}} &  \multirow{2}{*}{\includegraphics[width=0.05\textwidth]{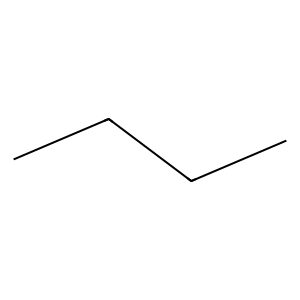}} &  \multirow{2}{*}{\includegraphics[width=0.05\textwidth]{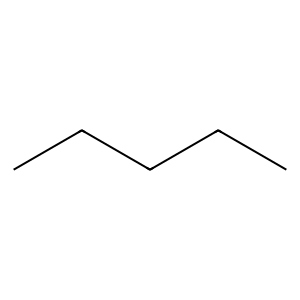}} &  \multirow{2}{*}{\includegraphics[width=0.05\textwidth]{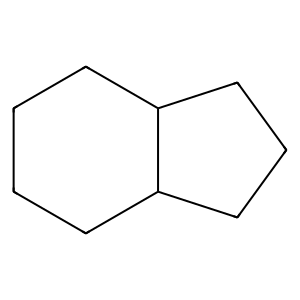}} &  \multirow{2}{*}{\includegraphics[width=0.05\textwidth]{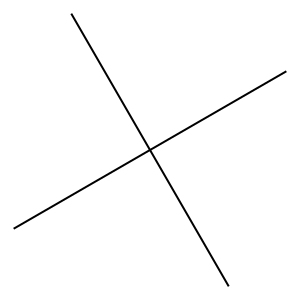}} &  \multirow{2}{*}{\includegraphics[width=0.05\textwidth]{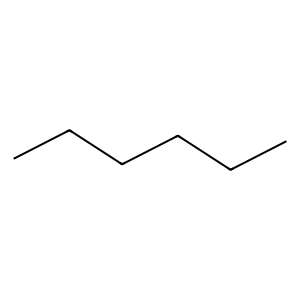}} &  \multirow{2}{*}{\includegraphics[width=0.05\textwidth]{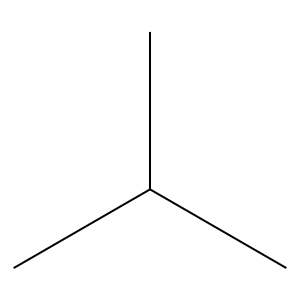}} &  \multirow{2}{*}{\includegraphics[width=0.05\textwidth]{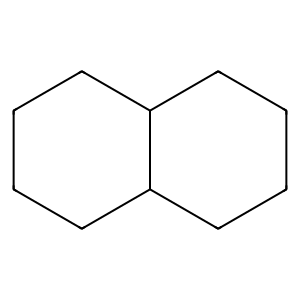}} &  \multirow{2}{*}{\includegraphics[width=0.05\textwidth]{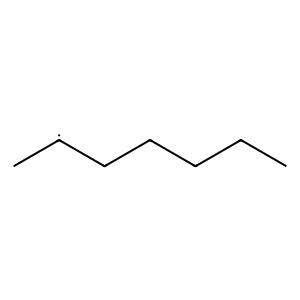}} &  \multirow{2}{*}{\includegraphics[width=0.05\textwidth]{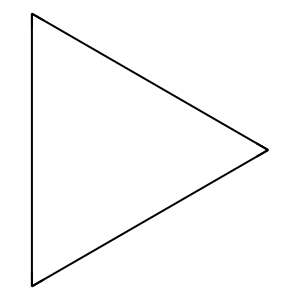}} &  \multirow{2}{*}{\includegraphics[width=0.05\textwidth]{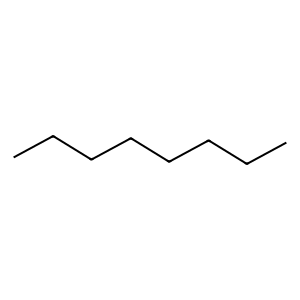}} &  \multirow{2}{*}{\includegraphics[width=0.05\textwidth]{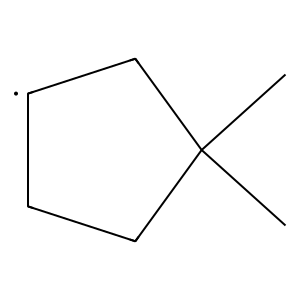}} &  \multirow{2}{*}{\includegraphics[width=0.05\textwidth]{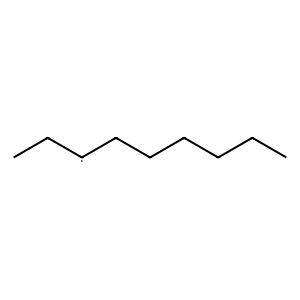}} &  \multirow{2}{*}{\includegraphics[width=0.05\textwidth]{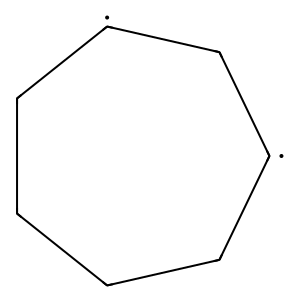}} &  \multirow{2}{*}{\includegraphics[width=0.05\textwidth]{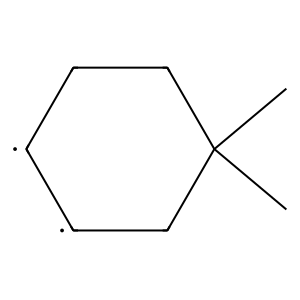}} &  \multirow{2}{*}{\includegraphics[width=0.05\textwidth]{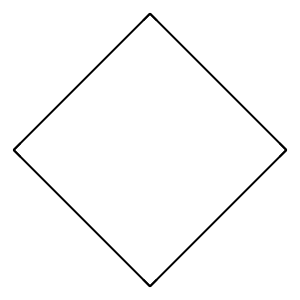}} &  \multirow{2}{*}{\includegraphics[width=0.05\textwidth]{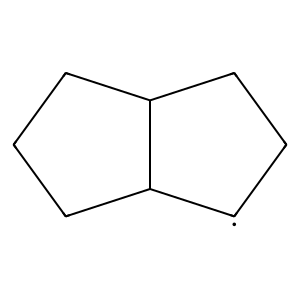}} &  \multirow{2}{*}{\includegraphics[width=0.05\textwidth]{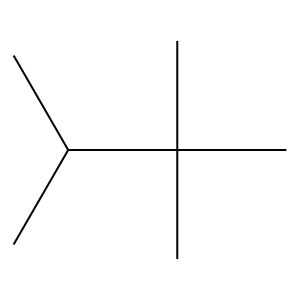}} &  \multirow{2}{*}{\includegraphics[width=0.05\textwidth]{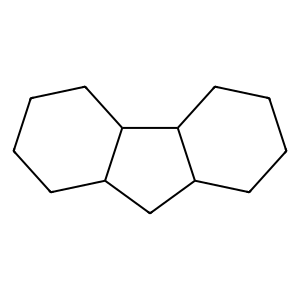}} &  \multirow{2}{*}{\includegraphics[width=0.05\textwidth]{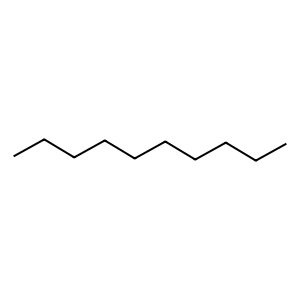}} &  \multirow{2}{*}{\includegraphics[width=0.05\textwidth]{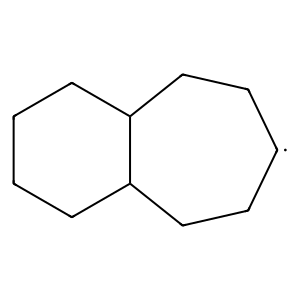}} &  \multirow{2}{*}{\includegraphics[width=0.05\textwidth]{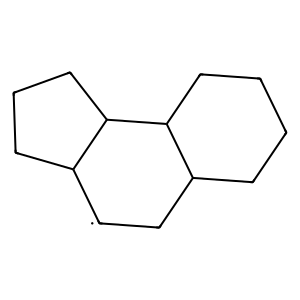}} &  \multirow{2}{*}{\includegraphics[width=0.05\textwidth]{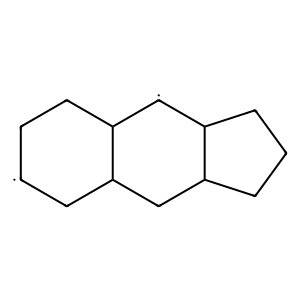}} &  \multirow{2}{*}{\includegraphics[width=0.05\textwidth]{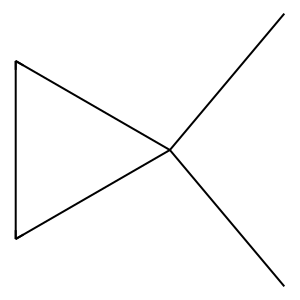}} &  \multirow{2}{*}{\includegraphics[width=0.05\textwidth]{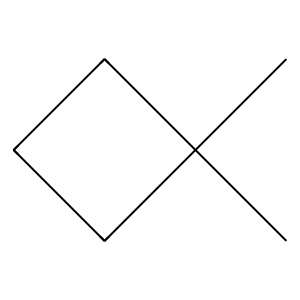}} &  \multirow{2}{*}{\includegraphics[width=0.05\textwidth]{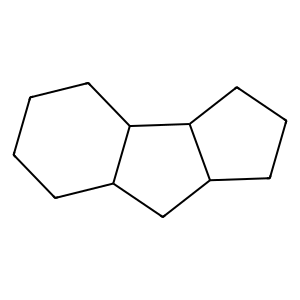}} \\
      &&&&&&&&&&&&&&&&&&&&&&&&&&&&&\\
\midrule
\textbf{Counts}   & 10000 &  7548 &                                                       6198 &                                                       5629 &                                                       3904 &                                                       2204 &                                                       1799 &                                                       1772 &                                                       1348 &                                                        1330 &                                                        1071 &                                                         741 &                                                         573 &                                                         375 &                                                         208 &                                                         204 &                                                         176 &                                                         156 &                                                         113 &                                                         113 &                                                          90&                                                          80 &                                                          77 &                                                          66 &                                                          54 &                                                          45 &                                                          37 &                                                          32 &                                                          31 \\
\textbf{Accuracy} &                                                          1.0 &                                                        0.997 &                                                        0.999 &                                                        0.986 &                                                         0.99 &                                                        0.969 &                                                        0.999 &                                                          1.0 &                                                        0.963 &                                                         0.997 &                                                         0.997 &                                                         0.933 &                                                         0.999 &                                                         0.954 &                                                         0.999 &                                                         0.988 &                                                         0.983 &                                                         0.982 &                                                         0.996 &                                                         0.995 &                                                         0.993 &                                                         0.991 &                                                         0.998 &                                                         0.995 &                                                         0.996 &                                                         0.996 &                                                         0.996 &                                                         0.998 &                                                         0.998 \\
\bottomrule
\end{tabular}
}
    \label{tab:substruct_counting}
\end{table}

\paragraph{Long-Range Interactions.}
Inspired by \citet{di_giovanni_over-squashing_2023}, we investigate the long-range capacity by training models to recover messages sent from a designated node to another node in a graph. As over-squashing is particularly prone to happen at bottlenecks \cite{topping_understanding_2022}, we are considering a graph consisting of two cycles connected by a bottleneck in the form of a path. This graph structure also matches many substructures found in molecules. 

\cref{fig:recovery} shows the recovery rates for a GNN using no fragmentation, a ring fragmentation, as used by CIN \cite{bodnar_weisfeiler_2022}, and our rings-paths fragmentation. A standard GNN is able to recover messages from the neighborhood perfectly, but the performance deteriorates with increasing distance. After the path bottleneck, the performance is equivalent to random guessing. A model with rings fragmentation exhibits a behavior very similar to the baseline. In particular, it performs worse on one node on the path. We assume that this happens due to the additional messages on the rings, creating even more messages that traverse the bottleneck. For our fragmentation with the higher-level graph, we observe that all nodes can recover the messages nearly perfectly. 
The results closely mirror the commute times on those graphs (see \cref{fig:commute} in \cref{app:further_experiments}), validating the theoretical work by \citet{di_giovanni_over-squashing_2023} that proposed the commute time as an indicator for over-squashing.

\begin{figure}[t]
\centering
\includegraphics[width=\linewidth]{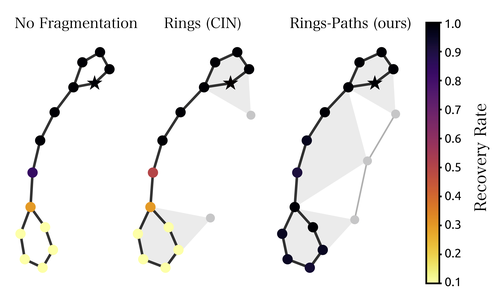}
\vskip -0.05in
\caption{Recovery rate of messages sent from the star node to all other nodes. A recovery rate of 0.1 corresponds to random guessing. The first graph has no fragmentation, the second one a rings fragmentation (like in CIN/CIN++), the third a rings and paths fragmentation (like in our model).}
\label{fig:recovery}
\end{figure}

\textbf{Predictive Performance.}
To evaluate the predictive performance on real-world molecular dataset, we use the long-range peptides  benchmark \cite{dwivedi_benchmarking_2022} and the large-scale molecular benchmark ZINC \cite{Sterling2015ZINC1}. 
We compare our model against standard GNNs, like \textbf{GCN} \cite{kipf2017semisupervised}, \textbf{GIN} \cite{xu_how_2019} or \textbf{GatedGCN} \cite{bresson2018residual}, higher-order GNNs such as \textbf{RingGNN} \cite{chen2023equivalence} or \textbf{3WLGNN} \cite{maron2019provably}, topological GNNs, especially \textbf{CIN} and \textbf{CIN++} \cite{bodnar_weisfeiler_2021, bodnar_weisfeiler_2022} and other fragment-biased GNNs such as \textbf{HIMP} \cite{fey_hierarchical_2020} and \textbf{GSN} \cite{bouritsas_improving_2023}. While our main focus lies on the evaluation against other message-passing GNNs, especially the group of fragment-biased GNNs, we compare against Transformer architectures for completeness. We test \textbf{Graphormer} \cite{ying_transformers_2021}, \textbf{GPS} \cite{rampasek_recipe_2023}, \textbf{GT} \cite{dwivedi_generalization_2021}, \textbf{SAN} \cite{kreuzer2021rethinking} and \textbf{GRIT} \cite{ma_graph_2023}. 

A summary of the used hyperparameters of our model and the experimental details for each experiment can be found in \cref{sec:experimental_details}. All our models adhere to the 500k parameter budget for both datasets. We do not use any additional feature augmentation, such as positional encodings.

To further investigate the long-range capabilities of our model, we empirically measure the performance on the Long-Range Graph Benchmark (LRGB), where the regression performance depends on these capabilities. The datasets we use are the LRGB Peptides-struct and Peptides-func \cite{dwivedi_benchmarking_2022}. The peptides datasets comprise 15535 peptides, each with an average of 150 nodes, that exhibit long amino acid chains with a lower average degree. As the molecules have a strong variability in their diameters, the models must generalize well to graphs of different sizes. These datasets are commonly utilized to benchmark the long-range capabilities of graph transformers against GNNs.

The peptides-struct task is a multi-label graph regression dataset where the objective is to predict various structural properties of the 3-dimensional molecules, including spatial lengths. Note that the node and edge features of the graphs do not contain any 3D information. Performance is evaluated using the mean absolute error (MAE). The peptides-func dataset consists of the same graphs, but requires the models to predict functional properties of the peptides. It is a multi-label classification task with 10 classes such as \textit{Antiviral} or \textit{Antibacterial}. The performance metric is the unweighted mean average precision (AP). 

The results presented in \cref{tab:peptides} demonstrate that our model achieves state-of-the-art performance for GNNs on both datasets. Notably, our model surpasses almost all graph transformers, even though transformers have an advantage in capturing long-range interactions due to their consideration of messages between all pairs of nodes, resulting in quadratic complexity. GRIT stands out as the only model that surpasses our performance, displaying exceptional results across all datasets.

Furthermore, we evaluate our model on ZINC \cite{Sterling2015ZINC1}, a collection of chemical compounds. The benchmark consists of a dataset that contains only a subset of 10,000 molecules and a dataset consisting of all 250,000 molecules. For both, we predict the penalized logP, which characterizes the drug-likeness of a molecule \cite{gómezbombarelli2017automatic}, and measure performance using MAE. 

\begin{table}[t]
    \centering
    \sisetup{detect-all}
    \caption{Predictive performance for multiple models on Peptides-struct and func. Best Transformer and best GNN are highlighted.}
    \vskip 0.05in
    \resizebox{\columnwidth}{!}{\begin{tabular}{
  @{}
  l
  l
  l
  l%
  l%
  l%
  l%
  l%
  l%
  @{}
}
\toprule
\multirow{3}{*}{\textbf{Type}} & \multirow{3}{*}{\textbf{Model}} & \multicolumn{2}{c}{\textbf{Peptides-}} \\
\cmidrule(lr){3-4}
 & & \textbf{Struct} & \textbf{Func} \\
 & & \text{(MAE $\downarrow$)} & \text{(AP $\uparrow$)} \\
 \midrule

\multirow{5}{*}{{Transformer}}
& GPS & 0.2500 $\pm$ 0.0012 & 0.6535 $\pm$ 0.0041\\
& SAN+LapPE& 0.2683 $\pm$ 0.0043 & 0.6384 $\pm$ 0.0121\\
& SAN+RWSE& 0.2545 $\pm$ 0.0012 & 0.6439 $\pm$ 0.0075\\
& GRIT & \textbf{0.2460} $\pm$ 0.0012 & \textbf{0.6988} $\pm$ 0.0082\\
\midrule
\midrule

\multirow{4}{*}{{Basic GNNs}} 
& GCN & 0.3496 $\pm$ 0.0013 & 0.5930 $\pm$ 0.0023\\
& GIN & 0.3547 $\pm$ 0.0045 & 0.5498 $\pm$ 0.0079\\
& GatedGCN & 0.3420 $\pm$ 0.0013 & 0.5864 $\pm$ 0.0035\\
& GatedGCN+RWSE & 0.3357 $\pm$ 0.0006 & 0.6069 $\pm$ 0.0035\\
\midrule

\multirow{1}{*}{{Topological}} 
 & CIN++ & 0.2523 $\pm$ 0.0013 & 0.6569 $\pm$ 0.0117 \\
\midrule

\multirow{2}{*}{Fragment-Biased} 
& HIMP & 0.2503 $\pm$ 0.0008 & 0.5668 $\pm$ 0.0149\\
& \ours{} (ours) & \textbf{0.2462} $\pm$ 0.0021 & \textbf{0.6678} $\pm$ 0.005\\
  
\bottomrule

\end{tabular}

}
    \vspace{-0.1in}
\label{tab:peptides}
\end{table}

As shown in \cref{tab:ZINC}, basic GNNs, which lack both high expressiveness and strong long-range capacity, exhibit significantly inferior performance. All fragment-biased approaches outperform the higher-order GNNs despite being less computationally demanding. Our model shows the best performance among all GNNs, with very similar results to CIN++ on ZINC-subset. In addition, our model outperforms most transformer architectures as well and achieves results comparable to those of GRIT on the full dataset.

\begin{table}[t]
    \sisetup{detect-all}
    \caption{Predictive performance for multiple models on ZINC 10k and ZINC full. Best Transformer and best GNN are highlighted.}
    \vskip 0.05in
    \resizebox{\columnwidth}{!}{\begin{tabular}{
  @{}
  l
  l
  l
  l%
  l%
  l%
  l%
  @{}
}
\toprule
\multirow{3}{*}{\textbf{Type}} & \multirow{3}{*}{\textbf{Model}} & \multicolumn{2}{c}{\textbf{ZINC}}\\
\cmidrule(lr){3-4}
 & &  \textbf{10k} & \textbf{Full}\\
 & & \text{(MAE $\downarrow$)}  & \text{(MAE $\downarrow$)}\\
 \midrule

\multirow{7}{*}{{Transformer}}
& Graphormer & 0.122 $\pm$ 0.006 & 0.052 $\pm$ 0.005\\
& GPS & 0.070 $\pm$ 0.006 & \text{-}\\
&  GT & 0.226 $\pm$ 0.014 & -\\
& SAN & 0.139 $\pm$ 0.006 & -\\
& Graphormer-URPE & 0.086 $\pm$ 0.007 & 0.028 $\pm$ 0.002\\
& Graphormer-GD & 0.081 $\pm$ 0.009 & 0.025 $\pm$ 0.004\\
& GRIT & \textbf{0.059} $\pm$ 0.002 & \textbf{0.023} $\pm$ 0.001 \\
\midrule
\midrule

\multirow{4}{*}{{Basic GNNs}} 
& GCN & 0.367 $\pm$ 0.011 & 0.113 $\pm$ 0.002\\
& GIN & 0.526 $\pm$ 0.051 & 0.088 $\pm$ 0.002\\
& GAT & 0.384 $\pm$ 0.007 & 0.111 $\pm$ 0.002\\
& GraphSAGE & 0.398 $\pm$ 0.002 & 0.126 $\pm$ 0.003\\
\midrule

\multirow{2}{*}{{Higher-order}} 
& RingGNN & 0.353 $\pm$ 0.019 & - \\
& 3WLGNN & 0.303 $\pm$ 0.068 & - \\
 \midrule

\multirow{2}{*}{{Topological}} & CIN-Small & 0.094 $\pm$ 0.004 & 0.044 $\pm$ 0.003\\    
 & CIN++ & \textbf{0.077} $\pm$ 0.004 & 0.027 $\pm$ 0.007\\
\midrule

\multirow{4}{*}{Fragment-Biased} 
& HIMP & 0.151 $\pm$ 0.006 & 0.036 $\pm$ 0.002\\
& GSN & 0.115 $\pm$ 0.012 & \text{-} \\
& Autobahn & 0.106 ± 0.004 & 0.029 ± 0.001 \\
& \ours{} (ours) & \textbf{0.0775} $\pm$ 0.005& \textbf{0.0237} $\pm$ 0.00\\
\bottomrule

\end{tabular}

}
\label{tab:ZINC}
\end{table}

\textbf{Generalization.}
To test the generalization capabilities of our model with the ordinal fragment encoding, we use a test set containing out-of-distribution molecules with completely unseen fragments. For this, we use the ZINC dataset and remove all molecules 
\begin{wraptable}{r}{4.5cm}
\vspace{-0.2in}
    \centering
    \sisetup{detect-all}
    \caption{We remove all molecules that contain 7-rings from the training set and test on all molecules, i.e., also the ones with 7-rings.}
    \vskip 0.05in
    \resizebox{4.5cm}{!}{\begin{tabular}{
  @{}
  l
  S[table-format=1.2]
  S[table-format=1.2]
  S[table-format=1.2]
  S[table-format=1.2]
  S[table-format=1.2]
  S[table-format=1.2]
  @{}
}
\toprule
 \multirow{3}{*}{\textbf{Model}} & \multicolumn{2}{c}{\textbf{ZINC 10k}} \\
 \cmidrule(lr){2-3}
 & \textbf{training} & \textbf{test} \\
 & \text{(MAE $\downarrow$)}  & \text{(MAE $\downarrow$)} \\
 \midrule
GRIT & 0.02 & 0.61 \\
 \midrule
\ours{} (ours) & 0.08 & \textbf{0.34} \\
\bottomrule
\end{tabular}
}
    \label{tab:ood}
\vspace{-0.1in}
\end{wraptable}
 containing a 7-ring from the training data. After training, we test on all molecules from the test set, thus also containing 7-rings that were not seen during training. 
The results in \cref{tab:ood} demonstrate that our model achieves an error 1.8 times lower than GRIT, showcasing the superior generalization capabilities.
Our better generalization capabilities can also be seen in the normal ZINC benchmark. In \cref{tab:rarity}, we group the ZINC dataset into groups based on the frequency of the rarest fragment. Our model outperforms HIMP everywhere and GRIT for graphs containing rare fragments.

\begin{table}[t]
    \centering
    \sisetup{detect-all}
    \vspace{-0.1in}
    \caption{Comparison of the MAE of GRIT and our model on ZINC-full. Graphs are grouped by the frequency of their rarest fragment.}    
    \vskip 0.05in
    \resizebox{\columnwidth}{!}{\begin{tabular}{
  @{}
  l|cccc
  @{}
}
\toprule
Frequency of  & \multirow{2}{*}{$<0.1\%$}& \multirow{2}{*}{$<1\%$}  & \multirow{2}{*}{$<10\%$} & \multirow{2}{*}{$\ge 10\%$} \\
rarest Fragment & & & & \\
\midrule
HIMP (MAE) & 14.4 & 0.48 & 0.15 & 0.030 \\
GRIT (MAE) & 9.5 & 0.26 & 0.026 & 0.018\\ 
\ours{} (ours) (MAE) & 5.3 & 0.15 & 0.045 & 0.021 \\
\bottomrule
\end{tabular}
}
    \label{tab:rarity}
\end{table}
Lastly, we test our model's capability to transfer the knowledge to a completely different dataset. We train on ZINC and predict the penalized logP on QM9 \cite{Wu2017MoleculeNetAB}. In \cref{tab:zero_shot_qm9}, we see that \ours{} achieves the lowest MAE of $1.12$, outperforming GRIT and HIMP, 
\begin{wraptable}{l}{3.5cm}
\vspace{-0.2in}
    \centering
    \sisetup{detect-all}
    \caption{Zero shot generalization to QM9}
    \vskip 0.05in
    \resizebox{4cm}{!}{\begin{tabular}{
  @{}
  l
  S[table-format=1.2]
  S[table-format=1.2]
  S[table-format=1.2]
  S[table-format=1.2]
  S[table-format=1.2]
  S[table-format=1.2]
  @{}
}
\toprule
 \multirow{1}{*}{\textbf{Model}} & {\textbf{QM9} (MAE $\downarrow$)} \\
 \midrule
HIMP & 3.43 \\
GRIT & 1.22 \\
\ours{} (ours) & \textbf{1.12} \\
\bottomrule
\end{tabular}
}
    \label{tab:zero_shot_qm9}
\end{wraptable}
suggesting that our model generalizes better to unseen data distributions due to our inductive bias and corresponding fragmentation. 
In summary, we showcased the generalization capabilities of our model on both a completely unseen dataset and a slightly shifted data distribution. The generalization capabilities also help our model perform better on rare fragments.

\section{Limitations}\label{sec:limitations}
While our method has strong predictive performance, it can further be improved: first, our fragmentation method and ordinal encoding is tailored towards molecules and not applicable to large densely connected graphs such as citation or social networks as one will find vast amount of meaningless fragments, introducing a lot of noise. Second, while our generalization experiments show superior performance over GRIT, \cref{tab:rarity} also shows that GRIT performs slightly better on molecules with frequent fragments. We leave it for future work to further improve fragment-biased models on those data.

\section{Conclusion}
\label{sec:conclusion}
In this work, we proposed a new expressivity measure, the Fragment-WL test, which provides a hierarchy on existing fragment-biased GNNs. Based on these insights, we proposed an expressive new model, which, together with our new fragmentation, outperforms all GNN approaches and most transformer architectures. Moreover, our model demonstrates predictive performance comparable to that of the best transformer model, while surpassing it in terms of generalization capabilities, despite having only linear complexity. This positions our approach as a robust solution for a variety of molecular modeling tasks. 

We believe that our work lays the foundation for several promising future directions. These include extending the expressivity hierarchy to, for example, incorporate orbit information \cite{bouritsas_improving_2023}. In addition, future work could seek to further improve the predictive performance on frequent data, or using fragment-biases in multi-task or meta-learning settings.

\section*{Acknowledgements}
This project is supported by the Bavarian Ministry of Economic Affairs, Regional Development and Energy with funds from the Hightech Agenda Bayern. Additionally, it is supported by the Helmholtz Association under the joint research school ``Munich School for Data Science - MUDS" and by the German Federal Ministry of Education and Research (BMBF) under Grant No. 01IS18036B.

\section*{Impact Statement}
Among other contributions, this work presents an approach for predicting the properties of molecules. In the area of machine learning for drug discovery, such methods can sometimes be used for harmful purposes. This also applies to our research, since it might help to discover or create dangerous substances. Despite these concerns, we believe that the benefits of our work outweigh the risks.

\bibliography{references}
\bibliographystyle{icml2024}

\newpage
\appendix
\onecolumn
\section{Proofs}\label{app:proofs}
This chapter presents the proofs for the theorems introduced in \cref{sec:theory}, and the expressiveness analysis of existing fragment-biased and topological GNNs. We will first introduce general concepts that will later help to bound the expressiveness of different models and our Fragment WL test. 
\subsection{Color refinement and expressiveness: Useful definitions and lemmas}
To prove the power of different graph coloring algorithms, it will be useful to first introduce the definition of color refinement. The intuition is that a "finer" coloring contains more information than a "corser" coloring.
\begin{definition}
Let $c,d$ be colorings of a graph $\graph$. The coloring $c$ refines $d$ (we write $c \sqsubseteq d$) if there exists a function $h$ such that $h(c_v) = d_v$ for all $v \in \vertices$.
\end{definition}
We will sometimes write $c_v \sqsubseteq d_v$ if $c \sqsubseteq  d$  and the set of vertices $\vertices$ is clear from the context.
\begin{example}
    Let $c^{(t)}$ be the coloring of iteration $t$ of the WL test. Then one can easily show that $c^{(t)}_v \sqsubseteq c^{(l)}_v$ for $l \le t$ as $c^{(t)}_v$ contains the information of all previous colorings $c^{(l)}_v$.
\end{example}
Note that an alternative definition of color refinement \cite{bodnar_weisfeiler_2021} is: $c \sqsubseteq d$ iff $c_v = c_w$ implies $d_v = d_w$ for all $v,w \in \vertices$. It is easy to see that the two definitions are equivalent.
Next, we extend our definition to arbitrary functions and not just colorings.
\begin{definition}
Let $a,b$ be two functions over the set of vertices $\vertices$. Then $a \sqsubseteq b$ if there exists a function $h$ such that $h(a(v)) = b(v)$ for all $v \in \vertices$.
\end{definition}
Intuitively, $a \sqsubseteq b$ means that we can compute $b(v)$ from the result of $a(v)$. So $a(v)$ contains more or the same information as $b(v)$. Again, we will sometimes simply write $a_v \sqsubseteq b_v$ with $a_v := a(v)$, $b_v := b(v)$ if $a \sqsubseteq b$ and if the set of vertices $\vertices$ is clear from the context.
\begin{example}
    Let $c^{(t)}$ be the coloring of iteration $t$ of the WL test. Then, because of the injectiveness of the hash function $\hash$ in \cref{eq:hash}:
    \begin{align*}
    c^{(t)}_v \sqsubseteq \ldbrace c^{(t-1)}_w \mid w \in \neighborhood(v) \rdbrace.
    \end{align*}
    Note that we use the simplified notation here. The right and left-hand side are actually the functions that map from $v \in \vertices$ to these terms. Intuitively, this shows that one can compute the previous color of all neighbors from the color of a node.
\end{example}

It is easy to see that the refinement relation is transitive, i.e., $a \sqsubseteq b$ and $b \sqsubseteq c$ imply $a \sqsubseteq c$.

We will now formally define the expressive power of an algorithm with respect to the ability to distinguish non-isomorphic subgraphs. 
\begin{definition}
A function $f$ is (in parts) \emph{more powerful} than a function $g$ if there exist two non-isomorphic graphs $\graph^1$, $\smash{\graph^2}$ such that $f$ can distinguish them \[\smash{f(\graph^1) \neq f(\graph^2)}\] whereas $g$ cannot distinguish them \[\smash{g(\graph^1) = g(\graph ^2)}.\] 
\end{definition}
Note that this relation is \emph{not} anti-symmetric, i.e. $f$ can be (in parts) more powerful than $g$, and conversely, $g$ can also be (in parts) more powerful than $f$. Hence, we introduce the following stronger anti-symmetric relation: 
\begin{definition}
A function $f$ is \emph{strictly more powerful} than $g$ (denoted as $f > g$) if 
\begin{enumerate}
    \item $f$ is more powerful than $g$
    \item and $g$ is not (in parts) more powerful than $f$.
\end{enumerate}
\end{definition}
Additionally, we write we write $g \le f$ if a function $g$ is not more powerful than $f$.

Next, we will prove a connection between color refinement and expressiveness: a function that always produces a finer coloring cannot be less powerful than a function with a coarser coloring.
\begin{lemma}
    \label{lem:refinement}
    Let $f,g$ be functions with $f \sqsubseteq g$ for all graphs. Then, $g$ is not more powerful than $f$, i.e., $g \le f$.
\end{lemma}
\begin{proof}
    Assume for the sake of contradiction that there exist non-isomorphic graphs $\graph^1, \graph^2$ that can be distinguished by $g$ but not by $f$.
    Let $d$ be the coloring obtained with $g$, and $c$ be the coloring obtained with $f$. The multiset of colors $d$ has to differ for $\graph^1$ and $\graph^2$, i.e. there exists a color $\alpha$ with 
    \begin{align*}
        D_\alpha^1 &:= \{ v \mid d_v = \alpha, \: v \in \vertices^1 \} \\
        D_\alpha^2 &:= \{ v \mid d_v = \alpha, \: v \in \vertices^2 \}
    \end{align*}
    such that
    \begin{align}
        \label{eq:diff}
        |D_\alpha^1| \neq |D_\alpha^2|.
    \end{align}
    Since $c$ refines $d$ no node $v$ in $D_\alpha^1$ and $D_\alpha^2$ can share a color $c_v$ with another node not in $D_\alpha^1$ and $D_\alpha^2$. Hence, the set of colors $c$ of nodes in $D_\alpha^1$ and $D_\alpha^2$ is disjoint from the set of colors $c$ for the other nodes in the graph. But because of \ref{eq:diff} there has to exist a color $\beta$ with
    \begin{align*}
        C_\beta^1 &:= \{ v \mid c_v = \beta, \: v \in \vertices^1 \}  \subseteq D_\alpha^1\\
        C_\beta^2 &:= \{ v \mid c_v = \beta, \: v \in \vertices^2 \} \subseteq D_\alpha^2
    \end{align*}
    and
    \begin{align*}
        |C_\beta^1| \neq |C_\beta^2|.
    \end{align*}
    This contradicts the initial assumption. 
\end{proof}

Note that all our augmentation functions introduced in \cref{sec:theory} only add information to the graph, or more formally:
\begin{definition}
A function $g$ from graphs to graphs is called additive if the set of nodes and edges does not decrease, i.e., $g(\vertices, \edges, \features) = (\vertices', \edges', \features')$ with  $\vertices \subseteq \vertices'$ and $\edges \subseteq \edges'$.
\end{definition}
But by adding too much information, one could completely destroy the initial structure of the graph (e.g., make every graph a complete graph). Hence, we need the additional condition that it should be possible to recover the original graph from the augmented graph.
\begin{definition}
An additive function $g$ from graphs to graphs (i.e, $g(\graph) = \graph'$) is called reversible if there exists a function $h$ for vertices $v \in \vertices'$ such that 
\begin{align}
    \label{eq:features}
    h(\features'_v)  = 
    \begin{cases}
       \features_v & v \in \vertices \\
       \bot & v \notin \vertices
    \end{cases}
\end{align}
 and a function $\varepsilon$ for edges $e = \{u,v\} \in \edges'$ such that
\begin{align}
\varepsilon(\features'_u, \features'_v) = 
\begin{cases}
\label{eq:edges}
1 & e \in \edges \\
0 & e \notin \edges.
\end{cases}
\end{align}
\end{definition}
By only adding such reversible information the WL test cannot become less powerful:
\begin{lemma}
    \label{lem:lower bound}
    Let $g$ be a reversible function. Then $g$-WL is not less powerful than WL.
\end{lemma}
\begin{proof}
    Let $c^{(t)}$ be the coloring obtained by the $t$-th iteration of the WL-test, and $d^{(t)}$  the coloring of $g$-WL.
    We will show by induction that there exists a function $h$ such that $h(d_v^{(t)}) = c_v^{(t)}$ for $v \in \vertices$, i.e $d_v^{(t)} \sqsubseteq c_v^{(t)}$. 
    For $t=0$, this follows immediately from \cref{eq:features}. 
    For the induction step, note 
    \begin{align*}
        d_v^{(t)} \sqsubseteq \ldbrace d_u^{(t-1)} \mid u \in \neighborhood_{\graph'}(v) \rdbrace.
    \end{align*}
    Now note that the function $\varepsilon$ (\cref{eq:edges}) makes it possible to reconstruct the neighborhood in $\graph$ based on the features $\features'$ and neighborhood in $\graph'$. Since $d^{(t)}$ only refines the features $\features'$, we can also reconstruct the neighborhood in $\graph$ based on $d^{(t)}$ and, hence,
    \begin{align*}
        \ldbrace d_u^{(t-1)} \mid u \in \neighborhood_{\graph'}(v) \rdbrace \sqsubseteq \ldbrace d_u^{(t-1)} \mid u \in \neighborhood_{\graph}(v) \rdbrace
    \end{align*}
    and by induction hypothesis
    \begin{align*}
        \ldbrace d_u^{(t-1)} \mid u \in \neighborhood_{\graph}(v) \rdbrace \sqsubseteq \ldbrace c_u^{(t-1)} \mid u \in \neighborhood_{\graph}(v) \rdbrace.
    \end{align*}
    So, taken together, we have
    \begin{align}   
    \label{eq:rev1}
        d_v^{(t)} \sqsubseteq \ldbrace c_u^{(t-1)} \mid u \in \neighborhood_{\graph}(v) \rdbrace.
    \end{align}
    Additionally, note that
    \begin{align}
        \label{eq:rev2}
        d_v^{(t)} \sqsubseteq d_v^{(t-1)} \stackrel{\text{IH}}{\sqsubseteq} c_v^{(t-1)}.
    \end{align}
    By combining \cref{eq:rev1} and \cref{eq:rev2}, we get
    \begin{align*}
    d_v^{(t)} \sqsubseteq \left ( c_v^{(t-1)},\: \ldbrace c_u^{(t-1)} \mid u \in \neighborhood_{\graph}(v) \rdbrace \right ) \sqsubseteq c_v^{(t)}.
    \end{align*}
     Hence, $d$ refines $c$, and by \cref{lem:refinement} $g$-WL is not less powerful than WL.
\end{proof}

Now that we have found a lower bound of the expressiveness, we will also give an upper bound of the expressiveness. The idea is that a graph augmentation function does not increase the power of a coloring algorithm $f$ (e.g., the WL test) if all the information added by $g$ can also be computed from the coloring obtained with $f$. Or, to put it differently, a graph augmentation function $g$ does not increase the expressiveness if it is possible to compute the set of colors on $g(G)$ from the set of colors on $G$.
\begin{lemma}
\label{lem:upper bound}
Let $g$ be a function that augments a graph, i.e., a function from graphs to graphs. Let $f$ be a coloring function. If there exists a function $h$ such that  
\begin{align*}
f(g(G)) = h(f(G))
\end{align*}
then $f \circ g$ is not more powerful than $f$.
\end{lemma}
\begin{proof}
   Let $G^1, G^2$ be two non-isomorphic graphs that are distinguishable by $f \circ g$. Then
   \begin{align*}
   f(g(G^1)) &\neq f(g(G^2)) \\
   h(f(G^1)) &\neq h(f(G^2)) 
   \end{align*}
   It follows that $f(G^1) \neq f(G^2)$. Hence, $f \circ g$ is not more powerful than $f$.
\end{proof}
Note that the condition in the lemma is similar to the definition of color refinement. However, we cannot use color refinement directly because the function $g$ could add or delete nodes, making a direct comparison of nodes between $G^1$ and $G^2$ impossible. Consequently, we have to use a more global view rather than the more localized approach of color refinement.

\subsection{Graph augmentation functions}
We will now analyze the change in expressiveness with some graph augmentation functions that model message-passing schemes that are frequently used in practice.

\subsubsection{Fragment augmentations}
We will first give the proofs for the augmentation functions from \cref{sec:theory} that incorporate fragment information.
\paragraph{Proof of Theorem \ref{theo:powerful}}
\theopowerful*
\begin{proof}
    Consider a fragmentation scheme that decomposes a graph into every possible subgraph (so the vocabulary is the set of all possible graphs). It is known that for every $k$ there exist two non-isomorphic graphs $\graph^1$ and $\graph^2$ that are indistinguishable by $k$-WL. But as $\graph^1$ and $\graph^2$ are themselves part of the vocabulary and the fragmentation, they are trivially distinguishable by \NF-WL, \FR-WL, and \HLG-WL. 
\end{proof}

\paragraph{Proof of Theorem \ref{theo:NF}}
\theoNF*
\begin{proof}
\citet{chen_can_2020} showed that the WL test cannot count the number of (induced) subgraphs with at least three nodes, i.e. for any substructure $S$ with more than three nodes, there exist non-isomorphic graphs $\graph^1$ and $\graph^2$ such that 2-WL cannot distinguish them but they have a different count of $S$. 
So, let $X$ be the substructure that $\mathcal{F}$ recovers. Then, WL cannot count the number of occurrences of $X$, whereas NF-WL can trivially count the number of $X$ in a graph.
Hence, NF-WL is more powerful than WL. Since NF is a reversible function, NF-WL  is by \cref{lem:lower bound} also not less powerful than WL, making NF-WL strictly more powerful than WL.
\end{proof}

\paragraph{Proof of Theorem \ref{theo:FR}}
Before coming to the proof of \cref{theo:FR}, we will prove the following useful lemma:
\begin{lemma}
\label{WL lemma}
    Two graphs $\graph^1 = (\vertices^1, \edges^1, \features^1)$, $\graph^2 = (\vertices^2, \edges^2, \features^2)$ are undistinguishable by WL if 
    the set of node features is the same
    \[ \features^1 = \features^2 := \features\]
    and all nodes with the same node feature have the same neighborhood
    \begin{align}
    &\forall i,j \in \vertices^1 \cup \vertices^2: \nonumber \\
    &\features_i = \features_j \Rightarrow \{\features_n \: |\: n \in \neighborhood(i)\} = \{\features_m \: |\: m \in \neighborhood(j)\}.
    \label{neighborhood}
    \end{align}
\label{lemma:WL}
\end{lemma}
\begin{proof}
We will show by induction over $t$ that the color of all nodes with the same node features is the same:
\begin{align*}
    &\forall i,j \in \vertices^1 \cup \vertices^2: \\
    &\features_i = \features_j \Rightarrow c_i^t = c_j^t.
    \end{align*}
For $t=0$, this follows immediately from $c^0_i = \hash(\features_i)$.
\\For $t>0$, we have for nodes $i,j$ with $X_i = X_j$
\begin{align*}
c^t_i &= \hash\bigl(c^{t-1}_i, \{c^{t-1}_n \mid n\in N(i)\}\bigr) \\
&= \hash\bigl(c^{t-1}_j, \{c^{t-1}_n \mid n \in N(i)\}\bigr) && \text{(by IH)} \\
&= \hash\bigl(c^{t-1}_j, \{c^{t-1}_m \mid m \in N(j)\}\bigr) && \text{(by \ref{neighborhood} and IH)} \\
&= c^t_j
\end{align*}
As both graphs have the same node features $\features^1 = \features^2$ the set of colors is also the same in each iteration $t$ of the \WL-test. Hence, the graphs are indistinguishable by the \WL-test
\end{proof}

\theoFR*

\begin{figure}[tb]
\includegraphics[width= 0.7\linewidth]{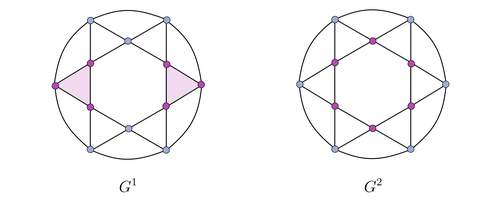}
\centering
\caption[Graphs indistinguishable by \NF-WL but distinguishable by \FR-WL]{Graph $\graph^1$ and graph $\graph^2$ that are indistinguishable by \NF-WL but distinguishable by \FR-WL. Node features are represented by the color of the nodes.}
\label{fig:FR}
\end{figure}

\begin{proof}
With Lemma \ref{lemma:WL} we can now prove Theorem \ref{theo:FR}: We first show that FR-WL is more powerful than NF-WL.
Consider the two graphs $G^1$, $G^2$ depicted in Figure \ref{fig:FR} with two different node features colored violet and blue. Note that when not considering the node features the graphs are identical and each node is isomorphic to every other node. As every fragmentation scheme has to be permutation invariant, every node is assigned the same additional node feature in NF (which, thus, holds no additional information to distinguish the graphs). Now, observe that the two graphs fulfill the conditions of Lemma \ref{WL lemma}. Therefore, the graphs are indistinguishable by NF-WL. In contrast, FR-WL can distinguish the two graphs as $\graph^1$ contains 3-cycles with three violet nodes whereas $\graph^2$ does not contain such 3-cycles. Hence, the fragment representation for these 3-cycles differs and distinguishes the two graphs.

Now, it remains to show that NF-WL is not more powerful than FR-WL. 
Let $c_v^{t}$, $d_v^{t}$ be the colorings of the $t$-th iteration of FR-WL and NF-WL, respectively.
We will show by induction over $t$ that $c^{t+1}_v \sqsubseteq d^{t}_v$.
To simplify the notation, let $\neighborhood_\uparrow(v) := \{ f \mid v \in f, \: f \in F\} $ be the set of fragments that $v$ is part of.
For $t=0$, note that after the first iteration of the FR-WL test each vertex $v \in G$ receives information about the fragment it is part of
\begin{align*}
c^1_v &= \hash\bigl(c^0_v, \ldbrace c^0_n \mid n \in \neighborhood_G(v) \uplus \neighborhood_\uparrow(v)\rdbrace \bigr) \\
 &= \hash\bigl(c^0_v, \ldbrace c^0_n \mid n \in \neighborhood(v) \rdbrace \uplus \ldbrace c^0_f \mid f \in \neighborhood_\uparrow(v)\rdbrace \bigr) \\
 &\sqsubseteq \bigl(c^0_v, \ldbrace c^0_f \mid f \in \neighborhood_\uparrow(v)\rdbrace \bigr) \\
 &= \bigl(\features_v, \ldbrace \text{type}(f) \mid f \in \neighborhood_\uparrow(v)\rdbrace \bigr) \\
 &\sqsubseteq d^0_v
\end{align*}
For the induction step $(t-1) \rightarrow t$, we have
\begin{align*}
c^t_v &= \hash\bigl(c^{t-1}_v, \ldbrace c^{t-1}_n \mid n \in \neighborhood_G(v) \uplus \neighborhood_\uparrow(v)\rdbrace \bigr) \\
 &\sqsubseteq \bigl(c^{t-1}_v, \ldbrace c^{t-1}_n \mid n \in \neighborhood_G(v)\rdbrace \bigr) \\
 &\sqsubseteq \bigl(d^{t-2}_v, \ldbrace d^{t-2}_n \mid n \in \neighborhood_G(v)\rdbrace \bigr) \\ 
 &\sqsubseteq d^{t-1}_v
 \end{align*}
This concludes the induction step. Hence, we have $c \sqsubseteq d$, and by \cref{lem:lower bound} NF-WL is not more powerful than FR-WL.
\end{proof}

\paragraph{Proof of Theorem \ref{theo:HLG}}
\theoHLG*
\begin{proof}
    Consider the two graphs $\graph^1$, $\graph^2$ depicted in Figure \ref{fig:HLG} with two different node features colored green and red. Note that the two graphs fulfill the conditions of Lemma \ref{lemma:WL}. Furthermore, even the graphs $\FR(\graph^1)$ and $\FR(\graph^2)$ fulfill these conditions. Hence, the graphs $\graph^1$, $\graph^2$ cannot be distinguished by \FR-WL. 
    
    In the higher level graph of $\graph^1$ each 3-cycle representation is connected to two other 3-cycle representations. Contrarily, in the higher level graph of $\graph^2$ each 3-cycle representation is connected to only one 3-cycle representations. Hence, the coloring will differ and \HLG-WL distinguishes the two graphs.

    Now, it remains to show that \FR-WL is not more powerful than \HLG-WL. This follows immediately from \cref{lem:lower bound} and the fact that the change in the graph from \FR{} to \HLG{} is reversible.
\end{proof}

\paragraph{Proof of Theorem \ref{theo:HLG power}}
\theoHLGpower*
\begin{proof}
    The proof follows a similar proof by \citet{bodnar_weisfeiler_2022}. Consider the Rook’s 4×4 and Shrikhande graph (both in the family of strongly regular graphs $\text{SR}(16, 6, 2, 2)$). The Shrikhande graph possesses 5-cycles while the Rook's graph does not. Hence, \HLG-WL can trivially disitnguish those two graphs with a fragmentation recovering 5-cycles. However, it is known that 3-WL cannot distinguish those two graphs \cite{bodnar_weisfeiler_2022}.
\end{proof}

\subsection{Additional graph augmentation function}
We will now consider two additional graph augmentation functions that are often used in practice: a learned representation for each edge and a learned representation for the complete graph that is connected to all other nodes. While these augmentations might be beneficial in practice, we will show that they do not increase expressiveness.

\paragraph{Edge representation}
We will first formally define the edge representation augmentation (ER). For every edge, we introduce a new node that is connected to its two endpoints.
\begin{definition}
    The edge representation graph augmentation function is $\ER(\vertices, \edges, \features) = (\vertices^\ER, \edges^\ER, \features^\ER)$ with 
    \begin{align*}
    &\vertices^\ER := \vertices \cup \edges, \\
    &\edges^\ER := \edges \cup \bigl\{ \{e ,v\} \mid e = \{u,v\} \in \edges \bigr\}\\
    &\features^\ER_i := \begin{cases} \features_i & i \in \vertices \\
    \alpha & i \in \edges 
    \end{cases}
    \end{align*}
    where $\alpha$ is some new label.
\end{definition}
Now, we can show that this augmentation does not increase expressiveness compared to the WL test.

\begin{lemma}
    \label{lem:ER}
    ER-WL is as powerful as WL.
\end{lemma}
\begin{proof}
    We will first show that ER-WL is not more powerful than WL:
    Let  $c^{(i)}, d^{(i)}$ be the colorings of the WL test and of the ER-WL test, respectively. Then we will show that

    We will use \cref{lem:upper bound} and give a function $h$ that maps a coloring $c^{(i)}$ of the WL test without edge representation to a coloring $d^{(i)}$ of ER-WL:
    Note that with the color $c^{(i)}_v$ of a node $v$ one can compute the colors $c^{(i-1)}_u$ of all neighboring nodes $u$. Hence, we can determine from $c^{(i)}$ the following multiset
    \begin{align*}
    \ldbrace (c^{(i)}_v, c^{(i-1)}_u) \mid e = (u,v) \in \edges \rdbrace
    \end{align*}
    which allows us to compute the corresponding edge representations $d_e^{(i)}$. 

    We will now show that ER-WL is not less powerful than WL.
    This follows directly from \cref{lem:lower bound} as ER is a reversible function.
\end{proof}

\paragraph{Graph representation}
We will now formally define the learned graph representation (GR), sometimes called virtual node.
\begin{definition}
    The graph rerpesentation augmentation function is $\GR(\vertices, \edges, \features) = (\vertices^\GR, \edges^\GR, \features^\GR)$ with 
    \begin{align*}
    &\vertices^\GR := \vertices \cup \{g\}, \\
    &\edges^\GR := \edges \cup \bigl\{ \{v, g\} \mid v \in \vertices\}\\
    &\features^\GR_i := \begin{cases} \features_i & i \in \vertices \\
    \alpha & i = g 
    \end{cases}
    \end{align*}
    where $\alpha$ is some new label.
\end{definition}
Similar to the edge representation, the graph representation does not increase expressiveness:
\begin{lemma}
    \label{lem:GR}
    GR-WL is as powerful as WL.
\end{lemma}

\begin{proof}
    We will first show that GR-WL is not more powerful than WL. We will use \cref{lem:upper bound} and give a function $h$ that maps a coloring $c^{(i)}$ of the WL test without graph representation to a coloring $d^{(i)}$ of GR-WL: We will show this by induction over $i$. For $i=0$, this follows immediately from the definition of GR. For the induction step, assume that there exists such a function from $c^{(t-1)}$ to $d^{(t-1)}$. 
    Note that the graph representation $d_g^{(t)}$ is computed as:
    \begin{align*}
    d_g^{(t)} &= \hash\bigl(d_g^{(t-1)}, \ldbrace d_v^{(t-1)} \mid v \in \neighborhood_{G'}(g) \rdbrace \bigr) \\
    &= \hash\bigl(d_g^{(t-1)}, \ldbrace d_v^{(t-1)} \mid v \in \vertices \rdbrace \bigl).
    \end{align*}
    which can be derived from $c^{(t-1)}$ by induction hypothesis.
    With this we can trivially compute $d_v^{(t)}$ from $c^{(t)}$ for all other nodes $v$, too.

    Since GR is a reversible function, by \cref{lem:lower bound} GR-WL is not less powerful than WL. Hence, GR-WL is as powerful as WL.
\end{proof}

\subsection{Expressiveness of existing models}
\label{app:expressive_existing_models}
\begin{table}[tb]
\centering
\caption{Overview of the vocabulary and expressiveness of existing topological and fragment-biased models. The bounds for GSN-v, $\mathcal{O}$-GNN, and HIMP are tight, i.e. when using a sufficient number of layers and injective neighborhood aggregators, the models are as powerful as the corresponding Fragment-WL test.}
\begin{tabular}{l|ll}
\toprule
\textbf{Model} & \textbf{Bounded by} & \textbf{Vocabulary} \\
\midrule
GSN-v\tablefootnote{We evaluate the GSN-v that is used in practice, i.e. with a vocabulary of cliques and rings. Note that the theory of the authors allows for potentially more expressive instantiations.} & $\le$ NF-WL & Cliques or Rings \\
\midrule
$\mathcal{O}$-GNN & $\le$ FR-WL & Rings \\
\midrule
HIMP & $\le$ HLG-WL & Rings \\
\midrule
\multirow{2}{*}{MPSN} & \multirow{2}{*}{$\le$ HLG-WL} & Simplicial complexes \\
& & (in practice cliques) \\
\midrule
\multirow{2}{*}{CIN} & \multirow{2}{*}{$\le$ HLG-WL} & CW complexes \\
& & (in practice rings \& edges) \\
\midrule
\multirow{2}{*}{CIN++} & \multirow{2}{*}{$\le$ HLG-WL} & CW complexes \\
& & (in practice rings \& edges) \\

\bottomrule
\end{tabular}

\label{tab:expressiveness}
\end{table}

We will use our Fragment-WL tests to compare the expressiveness of existing fragment-biased GNN models. \cref{tab:expressiveness} gives an overview of the vocabulary of existing fragment-biased and topological GNNs. Additionally, it shows the expressiveness in our Fragment-WL hierarchy.

\subsubsection{GSN-v} GSN-v \cite{bouritsas_improving_2023} incorporate fragment information as an additional node feature. The additional node features consist of the counts of fragment types a node is part of. Their framework also differentiates between different (non-symmetric)  positions inside the fragment (e.g., first node in path vs. second node in path) that correspond to different orbits. While their framework can use any fragmentation scheme, in all real-world experiments, they only use rings or cliques. Note that for rings and cliques, no different orbits exist, i.e., each node in the substructure has the same orbit. Hence, this information becomes irrelevant.
\begin{theorem}
    GSN-v using rings and/or cliques as vocabulary is at most as powerful as NF-WL. Additionally, when using injective neighborhood aggregators and a sufficient number of layers, GNSs are as powerful as NF-WL with a fragmentation scheme based on rings and cliques.
    \label{theo:GSN}
\end{theorem}
\begin{proof}
    GSN-v appends the node features by the counts of substructures and the respective orbits each node is part of. After that, a standard GNN is applied to the graph. 
    
    Note that in a ring or a clique, each node has exactly the same orbit. So, for a vocabulary based on cliques and rings the appended information degenerates to solely the substructure counts.
    Further, note that this substructure count function is an injective function $\lambda$ as defined in \cref{def:NF}. Hence, when using injective neighborhood aggregators and an MLP update function with a sufficiently large number of layers such that it can approximate the HASH function, GSN-v exactly models the NF-WL test. Hence,  GSN-v is exactly as powerful as NF-WL. 
\end{proof}

\subsubsection{$\mathcal{O}$-GNNs}
Besides representations for nodes, $\mathcal{O}$-GNNs \cite{zhu_mathcalo-gnn_2022} use explicit representation for rings, edges, and the whole graph. 
\begin{theorem}
    $\mathcal{O}$-GNNs \cite{zhu_mathcalo-gnn_2022} are at most as powerful as FR-WL. Additionally, when using injective neighborhood aggregators and a sufficient number of layers, $\mathcal{O}$-GNNs are as powerful as FR-WL with a fragmentation scheme based solely on rings.
    \label{theo:OGNN}
\end{theorem}
\begin{proof}
    $\mathcal{O}$-GNNs (when using injective neighborhood aggregators instead of their original sum aggregators and an MLP with a sufficiently large number of layers such that it can approximate the HASH function) models performing the WL test on an FR, ER, and GR augmented graph. As shown in \cref{lem:ER,lem:GR}, the edge representation and the graph representation do not influence the expressivity. Hence, $\mathcal{O}$-GNNs are exactly as powerful as FR-WL with a fragmentation scheme based solely on rings.
\end{proof}

\subsubsection{HIMP}
HIMP \cite{fey_hierarchical_2020} builds a higher-level junction tree based on rings and edges for message passing on the original graph, the higher-level junction tree, and between those two.

\begin{theorem}
    HIMP  is at most as powerful as HLG-WL. Additionally, when using injective neighborhood aggregators and a sufficient number of layers, HIMP is as powerful as HLG-WL with a fragmentation scheme based on rings and edges.
    \label{theo:HIMP}
\end{theorem}
\begin{proof}
    HIMP (when using injective neighborhood aggregators instead of their original sum aggregators and an MLP with a sufficiently large number of layers such that it can approximate the HASH function) exactly models performing the WL test on an HLG augmented graph.  Hence, HIMP is exactly as powerful as FR-WL with a fragmentation scheme based on rings and edges.
\end{proof}

\subsubsection{\ours{}}
Next, we consider the expressiveness of our \ours{} model:
\ourmodel*
\begin{proof}
For the proof, we rely on \cref{lem:ER}, the finding that an explicit edge representation does not augment expressiveness.
Notice that our model, when using injective neighborhood aggregators and an MLP with a sufficiently large number of layers such that it can approximate the HASH function,  exactly models performing the WL test on an HLG and ER augmented graph. As shown in \cref{lem:ER}, ER does not change the expressiveness. Hence, our model is exactly as powerful as HLG-WL.
\end{proof}

\subsubsection{Topological GNNs}
We will now consider topological GNNs. We will start by comparing HLG-WL with CWL, a variant of the WL test operating on CW complexes \cite{bodnar_weisfeiler_2022}. In the CWL framework, every graph is (permutation invariantly) mapped to a set of cells $\mathcal{X}$, a CW complex (using a skeleton-preserving lifting map). Let $\mathcal{X}_i$ denote the set of cells with dimension $i$. Then, $\mathcal{X}_0$ corresponds to all vertices $\vertices$ and $\mathcal{X}_1$ to all edges $\edges$. For higher dimensions, the results depend on the particular cellular lifting map. For instance, $\mathcal{X}_2$ could correspond to all cycles. 
 \begin{restatable}{theorem}{CWL}
    HLG-WL  is not less powerful than CWL, with a fragmentation scheme $\mathcal{F}$ that corresponds to the cellular lifting map used by CWL.
\end{restatable}
\begin{proof}
    Let $F$ be the fragmentation that corresponds to $\mathcal{X}$ without the vertices: $F = \mathcal{X} \setminus \vertices$.
    Let $c^{(t)}, b^{(t)}$ be the coloring of iteration $t$ of HLG-WL and CWL, respectively. We will show that $b^{(t)} \sqsubseteq c^{(2t)}$ which implies that $b \sqsubseteq c$ after termination.

    We will show this by induction over $t$. For $t=0$, this follows immediately from the fact that the node features in HLG-WL are finer than the features of cells in CWL.

    Now we will show $c^{(t)} \sqsubseteq b^{(2t)}$ assuming $c^{(t-1)} \sqsubseteq b^{(2t-2)}$:

    The idea of the induction step is that the hash update function $\hash$ receives more information in HLG-WL compared to CWL. Let us first consider vertices, i.e., $\mathcal{X}_0$.
    The update function in CWL for $v \in \vertices = \mathcal{X}_0$ is:
    \begin{align*}
        b_v^{(t)} &= \hash\bigl( b_v^{(t-1)}, \ldbrace (b_w^{(t-1)}, b_e^{(t-1)}) \mid e = \{v,w\} \in \edges \rdbrace \bigr )
    \end{align*}
    Now note that the update function in HLG-WL for $v \in \vertices$ is:
    \begin{align*}
        c_v^{(2t)} &= \hash\bigl(c_v^{(2t-1)}, \ldbrace c_w^{(2t-1)} \mid  w \in \neighborhood_{\HLG(G)}(v) \rdbrace \bigr) \\
        &\sqsubseteq \hash\bigl(c_v^{(2t-1)}, \ldbrace c_e^{(2t-1)} \mid  e = \{v,w\} \in \edges  \rdbrace \bigr) \\
        &\sqsubseteq \hash\bigl(c_v^{(2t-1)}, \ldbrace (c_e^{(2t-2)}, c_w^{(2t-2)}) \mid  e = \{v,w\} \in \edges  \rdbrace \bigr) \\
        & \sqsubseteq \hash\bigl(b_v^{(t-1)}, \ldbrace (b_e^{(t-1)}, b_w^{(t-1)}) \mid  e = \{v,w\} \in \edges  \rdbrace \bigr) \\
        &= b_v^{(t)}
    \end{align*}
    The first step follows from $e \in \neighborhood_{\HLG(G)}(v)$ as the edges are part of the fragmentation $F$. The second step follows from $c_e^{(2t-1)} \sqsubseteq (c_e^{(2t-2)}, c_w^{(2t-2)}, c_v^{(2t-2})$. The third step uses the induction hypothesis.

    Now, we will consider a cell $x \in \mathcal{X}_k \subseteq F$. The update function in CWL is
    \begin{align*}
        b_x^{(t)} &= \hash\bigl( b_x^{(t-1)}, \ldbrace (b_u^{(t-1)}, b_o^{(t-1)}) \mid x \prec u,\: o \prec u  \rdbrace, \ldbrace b_l^{(t-1)} \mid l \prec x \rdbrace \bigr )
    \end{align*}
    where  $x \prec y$ means that $x$ with dimension $k$ is part of the cell $y$ of dimension $k+1$. For example, $e \prec r$ if $e$ is an edge in a ring $r \in X_{k+1}$. For details, we refer to \citet{bodnar_weisfeiler_2021}.

    The update function in HLG-WL for a fragment $x \in F \cap \mathcal{X}_k$ in $G' := \HLG(G)$ is:
    \begin{align*}
        c_x^{(2t)} &= \hash\bigl(c_x^{(2t-1)}, \ldbrace c_w^{(2t-1)} \mid  w \in \neighborhood_{G'}(x) \rdbrace \bigr) \\
        &\sqsubseteq \hash\bigl(c_x^{(2t-1)}, \ldbrace c_u^{(2t-1)} \mid  u \in \neighborhood_{G'}(x) \cap \mathcal{X}_{k+1}\rdbrace, \ldbrace c_l^{(2t-1)} \mid  l \in \neighborhood_{G'}(x) \cap \mathcal{X}_{k-1}\rdbrace \bigr)  \\
        &\sqsubseteq \hash\bigl(c_x^{(2t-1)}, \ldbrace c_u^{(2t-1)} \mid  x \prec u \rdbrace, \ldbrace c_l^{(2t-1)} \mid  l \prec x \rdbrace \bigr)  \\
        &\sqsubseteq \hash\bigl(c_x^{(2t-1)}, \ldbrace (c_u^{(2t-2)}, c_o^{(2t-2)}) \mid  x \prec u, \: o \prec u \rdbrace, \ldbrace c_l^{(2t-1)} \mid  l \prec x \rdbrace\bigr)  \\
        &\sqsubseteq \hash \bigl(b_x^{(t-1)}, \ldbrace (b_u^{(t-1)}, b_o^{(t-1)}) \mid  x \prec u, \: o \prec u \rdbrace, \ldbrace b_l^{(t-1)} \mid  l \prec x \rdbrace\bigr) \\
        &= b_x^{(t)}
    \end{align*}
    The steps are very similar to the vertex case above. This concludes the proof that $c$ refines $b$. By \cref{lem:refinement} this implies that HLG-WL is not less powerful than CWL using a fragmentation that corresponds to the cellular complex.
\end{proof}
As CWL bounds the expressiveness of CIN \cite{bodnar_weisfeiler_2022} and CIN++ \cite{giusti_cin_2023}, we get the following corollary:
\begin{corollary}
CIN and CIN++ are at most as powerful as HLG-WL with a fragmentation scheme that corresponds to the cellular lifting map.
\label{cor:CIN}
\end{corollary}
Additionally, CWL subsumes the WL version, SWL, introduced by \citet{bodnar_weisfeiler_2021} for simplicial complexes. The cellular complex just corresponds to all cliques of the graph. 
\begin{restatable}{corollary}{SWL}
HLG-WL, with a fragmentation scheme recovering cliques, is not less powerful than SWL.
\end{restatable}
As MPSNs \cite{bodnar_weisfeiler_2021} are bounded by SWL, we have the following result for MPSNs:
\begin{corollary}
MPSNs are at most as powerful as HLG-WL with a fragmentation scheme based on cliques.
\label{cor:MPSN}
\end{corollary}

\section{Further experiments}\label{app:further_experiments}

\subsection{Long-Range tests}
We provide more experiments to measure the long-range capabilities of \ours{}.

\textbf{Commute time}

In addition to the recovery rates in \cref{fig:recovery}, we also consider commute times.  The commute time between the nodes $a$ and $b$ is the expected time for a random walker from $a$ to reach $b$ and return again to $a$. \citet{di_giovanni_over-squashing_2023} have proposed the commute time as a measure for over-squashing. To compute and compare commute times across different fragmentations, we connected all nodes in each fragmentation that could exchange a message within one layer. \cref{fig:commute} shows the commute from the star node to every other node for the same graph as in \cref{fig:recovery}. The close alignment between commute time and recovery rate supports the theoretical findings by \citet{di_giovanni_over-squashing_2023} and further emphasizes the potentially enhanced long-range capabilities of our model. Additionally, we also compute commute times on a molecule from the ZINC dataset that contains more fragments (see \cref{fig:commute_larger}). 
\begin{figure}[tb]
\includegraphics[width=.6\linewidth]{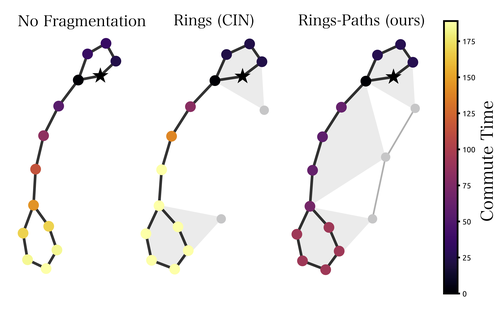}
\centering
\caption{Commute time from the star node to all other nodes. The first graph has no fragmentation, the second one a rings fragmentation (like in CIN/CIN++), the third a rings and paths fragmentation (like our model).}
\label{fig:commute}
\end{figure}

\begin{figure}[htb]
\includegraphics[width=.7\linewidth]{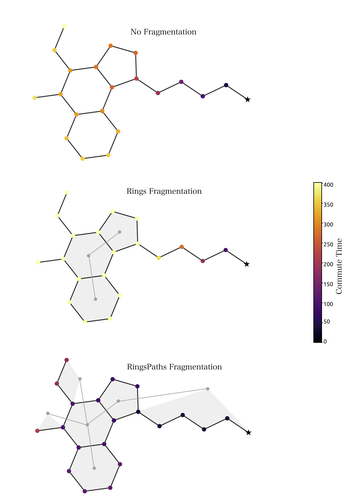}
\centering
\caption{Commute time from the star node to all other nodes. The first graph has no fragmentation, the second one a rings fragmentation (like in CIN/CIN++), the third a rings and paths fragmentation (like our model).}
\label{fig:commute_larger}
\end{figure}

Quantitatively, we compute average commute times on a random sample of molecules from the peptides dataset for a model without any fragmentation and for \ours{} (RingsPaths fragmentation with a higher-level graph). We observe that the addition of a higher-level graph reduces commute times by $16\%$. 

\begin{table}[h]
    \centering
    \caption{Average commute times between all nodes on a random sample of 50 molecular graphs from the peptides dataset with and without the higher-level graph.}
    \begin{tabular}{l|c}
    \toprule
        Normal Molecular Graph &  5056\\
        Molecular Graph + HLG & 4253\\
    \bottomrule
    \end{tabular}
    \label{tab:commute_time_preptides}
\end{table}

\subsection{Distribution of Fragments}
\cref{fig:frag_distribution}
 illustrates the distribution of fragment sizes, i.e., path lengths and ring lengths, extracted by our RingsPaths fragmentation method across the ZINC-10k, ZINC-full, and peptides datasets. It is worth noting that the peptides dataset features some exceptionally large rings.

\begin{figure}[tb]
\centering
\includegraphics[width=0.6\linewidth]{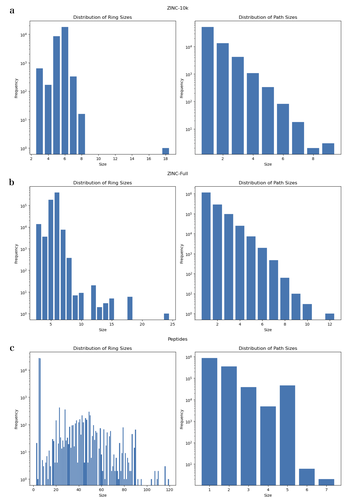}
\caption{Distribution of sizes of path and ring fragments for the a) ZINC-10k, b) ZINC-full, and c) Peptides dataset.}
\label{fig:frag_distribution}
\end{figure}

\subsection{Ablation Studies}
In the following, we test the design choices of our model and fragmentation. First, we test \ours{} without the different fragment information or ordinal encoding on ZINC and Peptides. We show the results in \cref{tab:ablation_model}. We observe that a reduction in expressiveness generally leads to a reduction in performance. An exception is the use of fragment representations (FR-WL, i.e. \ours{} - Higher-level graph) which shows a higher error on ZINC 10k and Peptides Struct. This is similar to the pattern that we observe in the message reconstruction toy experiment we show in \cref{fig:recovery} where the additional fragment representations increase the importance of the current substructure and do not contain a message from other parts of the molecule.
\begin{table}[h]
    \centering
    \caption{Ablation of different expressivity choices for \ours{}. Additionally, we ablate the ordinal encoding.}
    \begin{tabular}{l|cccc}
\toprule
\multirow{2}{*}{\textbf{Model}} &  \multicolumn{1}{c}{\textbf{ZINC}} & \multicolumn{2}{c}{\textbf{Peptides}} \\
\cmidrule(lr){2-2}\cmidrule(lr){3-4}
&  \textbf{10k} (MAE $\downarrow$) & \textbf{Struct} (MAE $\downarrow$) & \textbf{Func} (AP $\uparrow$) \\
\midrule
\ours{} (=HLG-WL) & \textbf{0.0775} $\pm$ 0.004 & \textbf{0.246} $\pm$ 0.002 & \textbf{0.668} $\pm$ 0.003\\
$-$ Higher-level graph (=FR-WL) & 0.0872 $\pm$ 0.004 & 0.256 $\pm$ 0.003 & 0.661 $\pm$ 0.005\\
$-$ Fragment representation (=NF-WL) & 0.0994 $\pm$ 0.007 & 0.247 $\pm$ 0.003 & 0.654 $\pm$ 0.005\\
$-$ All fragment information (=WL) & 0.1609 $\pm$ 0.003 & 0.249 $\pm$ 0.001 & 0.652 $\pm$ 0.005\\
\midrule
\ours{} $-$ ordinal encoding & 0.0945 $\pm$ 0.006 & 0.249 $\pm$ 0.001 & 0.666 $\pm$ 0.004\\
\bottomrule
\end{tabular}

    \label{tab:ablation_model}
\end{table}

Furthermore, we compare different fragmentation schemes in combination with \ours{}. In \cref{tab:ablation_fragmentation}, we observe that our RingsPath fragmentation scheme performs the best across the different datasets. 
\begin{table}[h]
    \centering
    \caption{Performance of \ours{} with different fragmentation schemes.}
    \begin{tabular}{l|cccc}
\toprule
\multirow{2}{*}{\textbf{Fragmentation Scheme}} &  \multicolumn{1}{c}{\textbf{ZINC}} & \multicolumn{2}{c}{\textbf{Peptides}} \\
\cmidrule(lr){2-2}\cmidrule(lr){3-4}
&  \textbf{10k} (MAE $\downarrow$) & \textbf{Struct} (MAE $\downarrow$) & \textbf{Func} (AP $\uparrow$) \\
\midrule
BBB &  0.127 & 0.252 $\pm$ 0.002 & 0.637 $\pm$ 0.003\\
BRICS & 0.127 & 0.247 $\pm$ 0.008 & 0.658 $\pm$ 0.011\\
Magnet & 0.098& - & - \\
\midrule
Rings &  0.078 & 0.249 $\pm$ 0.001 & 0.659 $\pm$ 0.007\\
RingsPaths (ours) & \textbf{0.077} & \textbf{0.246} $\pm$ 0.002 & \textbf{0.668} $\pm$ 0.005\\
\bottomrule
\end{tabular}

    \label{tab:ablation_fragmentation}
\end{table}

In \cref{fig:frag_comparion}, we look at how large the vocabulary size has to be per fraction of fragmented atoms. That is, for an increasing vocabulary, we observce how many atoms belong to a fragment. The steeper the increase the better. We can observe that on ZINC-10k BBB, BRICS and Rings are not able to assign a fragment to each atom no matter how large the vocabulary size. Magnet achieves full fragmentation but slower compared to our RingsPaths which is the most vocabulary efficient. We further show the necessary vocabulary size on ZINC-Full and Peptides for RingsPaths in \cref{tab:vocab_size}.

\begin{figure}[tb]
\centering
\includegraphics[width=0.5\linewidth]{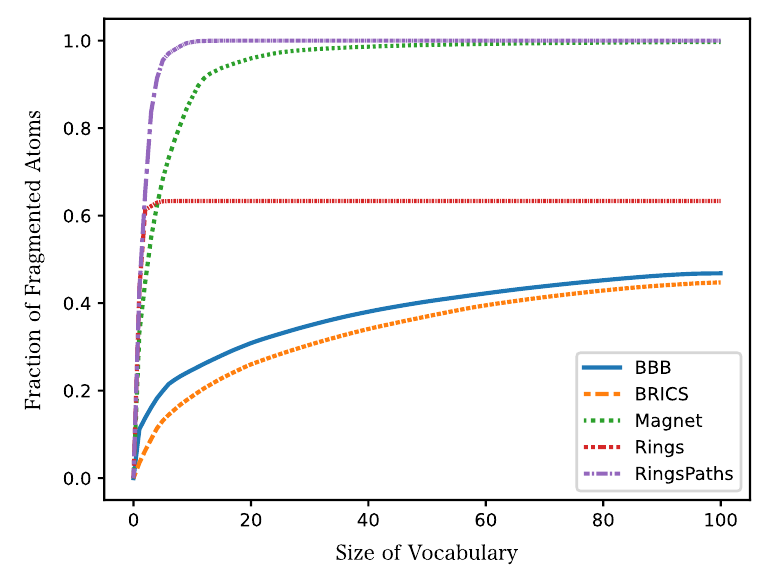}
\caption[Fragmented Atoms by Vocabulary Size]{Fraction of atoms in ZINC-10k dataset that are part of a fragment as a function of vocabulary size. A fraction of $1$ indicates that all molecules in the dataset can be completely fragmented. We compare the chemically inspired fragmentation schemes BBB, BRICS, and MagNet with a fragmentation based just on rings and our RingsPaths fragmentation. 
The substructures in the vocabulary are sorted by the frequency in which they appear in the molecules.}
\label{fig:frag_comparion}
\end{figure}

\begin{table}[h]
    \centering
    \caption{Vocabulary sizes for RingsPaths on different datasets.}
    \begin{tabular}{l|ccc}
\toprule
& \textbf{ZINC-10k} & \textbf{ZINC-Full} & \textbf{Peptides}\\
\midrule
\textbf{Vocabulary Size} & 18 & 28 & 100 \\
\bottomrule
\end{tabular}

    \label{tab:vocab_size}
\end{table}

\section{Experimental details}
\label{sec:experimental_details}
In the following, we will describe details for all our experiments.
Unless otherwise stated, for our \ours{}, we use a 2-layer fully connected neural network with ReLU activations and batch norm as the \MLP{} update function. For the aggregation method \AGG{}, we use a sum aggregation for messages within the original or higher-level graph and a mean aggregation for messages between the original graph and the higher-level graph. For training, we use the AdamW \cite{Loshchilov2017DecoupledWD} optimizer and gradient clipping with a value of 1. The model has been implemented in PyTorch \cite{pytorch} using the PyTorchGeometric \cite{pyg} and the PyTorch Lightning \cite{lightning} library. It is in parts adapted from HIMP \cite{fey_hierarchical_2020}.
All results for other methods are taken from \citet{rampasek_recipe_2023} and \citet{giusti_cin_2023}.

\subsection{ZINC and peptides}
The hyperparameters of our model for ZINC (10k and full) and peptides (struct and func) can be found in \ref{tab:hyper}. Note that we adhere to the 500K parameter budget. Each experiment is repeated over three different seeds except for the ZINC-full experiment, where we only have a single run because of computational and time limitations. 

\begin{table}
\centering
\begin{tabular}{l|llll}
\toprule
{} & peptides-struct & peptides-func & ZINC-10k & ZINC-full \\
\midrule
num\_layers      &               3 &             2 &        5 &         3 \\
hidden\_channels &             110 &           128 &       64 &       120 \\
num\_layers\_out  &               3 &             3 &        3 &     2 \\
frag-reduction  &             sum &           sum &      max &       max \\
out-reduction & mean & mean & mean & mean \\
dropout         &            0.05 &          0.15 &        0 &         0 \\
lr              &           0.001 &         0.001 &    0.001 &     0.001 \\
weight decay    &               0 &             0 &    0.001 &      0 \\
ema decay       &            0.99 &          0.99 &     0.99 &      0.99 \\
scheduler      &            Cosine &          ReduceonPlateau &      Cosine &       ReduceonPlateau \\
patience        &            - &            30 &     - &        15 \\
factor          &            - &           0.5 &     - &       0.9 \\
batchsize      &              32 &           128 &       32 &       128 \\
max epochs      &             300 &           400 &     2000 &      1000 \\
num parameters & 440K & 440K & 221K & 494K \\
\bottomrule
\end{tabular}

\caption{Hyperparameter configuration of our model for the ZINC and peptides benchmarks}
\label{tab:hyper}
\end{table}

\subsection{Expressiveness}
\begin{table}
\caption{Fragment counts for the 42 most common MagNet fragments in ZINC and accuracy scores of our model in predicting the counts.}
\resizebox{\textwidth}{!}{
\begin{tabular}{l|llllllllllllllHHHHHHHHHHHHHHH}
\toprule
     \multirow{2}{*}{\textbf{Fragment}}&  
     \multirow{2}{*}{\includegraphics[width=0.04\textwidth]{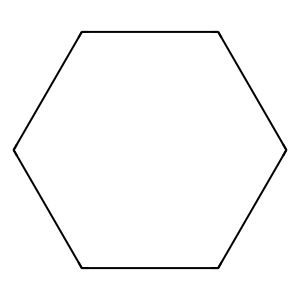}} &  \multirow{2}{*}{\includegraphics[width=0.04\textwidth]{{figures/mols/2}}} &  \multirow{2}{*}{\includegraphics[width=0.04\textwidth]{{figures/mols/3}}} &  \multirow{2}{*}{\includegraphics[width=0.04\textwidth]{{figures/mols/4}}} &  \multirow{2}{*}{\includegraphics[width=0.04\textwidth]{{figures/mols/5}}} &  \multirow{2}{*}{\includegraphics[width=0.04\textwidth]{{figures/mols/6}}} &  \multirow{2}{*}{\includegraphics[width=0.04\textwidth]{{figures/mols/7}}} &  \multirow{2}{*}{\includegraphics[width=0.04\textwidth]{{figures/mols/8}}} &  \multirow{2}{*}{\includegraphics[width=0.04\textwidth]{{figures/mols/9}}} &  \multirow{2}{*}{\includegraphics[width=0.04\textwidth]{{figures/mols/10}}} &  \multirow{2}{*}{\includegraphics[width=0.04\textwidth]{{figures/mols/11}}} &  \multirow{2}{*}{\includegraphics[width=0.04\textwidth]{{figures/mols/12}}} &  \multirow{2}{*}{\includegraphics[width=0.04\textwidth]{{figures/mols/13}}} &  \multirow{2}{*}{\includegraphics[width=0.04\textwidth]{{figures/mols/14}}} &  \multirow{2}{*}{\includegraphics[width=0.04\textwidth]{{figures/mols/15}}} &  \multirow{2}{*}{\includegraphics[width=0.04\textwidth]{{figures/mols/16}}} &  \multirow{2}{*}{\includegraphics[width=0.04\textwidth]{{figures/mols/17}}} &  \multirow{2}{*}{\includegraphics[width=0.04\textwidth]{{figures/mols/18}}} &  \multirow{2}{*}{\includegraphics[width=0.04\textwidth]{{figures/mols/20}}} &  \multirow{2}{*}{\includegraphics[width=0.04\textwidth]{{figures/mols/19}}} &  \multirow{2}{*}{\includegraphics[width=0.04\textwidth]{{figures/mols/21}}} &  \multirow{2}{*}{\includegraphics[width=0.04\textwidth]{{figures/mols/22}}} &  \multirow{2}{*}{\includegraphics[width=0.04\textwidth]{{figures/mols/23}}} &  \multirow{2}{*}{\includegraphics[width=0.04\textwidth]{{figures/mols/24}}} &  \multirow{2}{*}{\includegraphics[width=0.04\textwidth]{{figures/mols/25}}} &  \multirow{2}{*}{\includegraphics[width=0.04\textwidth]{{figures/mols/26}}} &  \multirow{2}{*}{\includegraphics[width=0.04\textwidth]{{figures/mols/27}}} &  \multirow{2}{*}{\includegraphics[width=0.04\textwidth]{{figures/mols/28}}} &  \multirow{2}{*}{\includegraphics[width=0.04\textwidth]{{figures/mols/29}}} \\
      &&&&&&&&&&&&&&&&&&&&&&&&&&&&&\\
\midrule
\textbf{Count}   & 12862 &  7548 &                                                       6198 &                                                       5629 &                                                       3904 &                                                       2204 &                                                       1799 &                                                       1772 &                                                       1348 &                                                        1330 &                                                        1071 &                                                         741 &                                                         573 &                                                         375 &                                                         208 &                                                         204 &                                                         176 &                                                         156 &                                                         113 &                                                         113 &                                                          90&                                                          80 &                                                          77 &                                                          66 &                                                          54 &                                                          45 &                                                          37 &                                                          32 &                                                          31 \\
\textbf{Accuracy} &                                                          1.0 &                                                        0.997 &                                                        0.999 &                                                        0.986 &                                                         0.99 &                                                        0.969 &                                                        0.999 &                                                          1.0 &                                                        0.963 &                                                         0.997 &                                                         0.997 &                                                         0.933 &                                                         0.999 &                                                         0.954 &                                                         0.999 &                                                         0.988 &                                                         0.983 &                                                         0.982 &                                                         0.996 &                                                         0.995 &                                                         0.993 &                                                         0.991 &                                                         0.998 &                                                         0.995 &                                                         0.996 &                                                         0.996 &                                                         0.996 &                                                         0.998 &                                                         0.998 \\
\end{tabular}
}
\resizebox{\textwidth}{!}{
\begin{tabular}{l|HHHHHHHHHHHHHHllllllllllllllH}
\toprule
    \multirow{2}{*}{\textbf{Fragment}}
      &  \multirow{2}{*}{\includegraphics[width=0.04\textwidth]{{figures/mols/1}}} &  \multirow{2}{*}{\includegraphics[width=0.04\textwidth]{{figures/mols/2}}} &  \multirow{2}{*}{\includegraphics[width=0.04\textwidth]{{figures/mols/3}}} &  \multirow{2}{*}{\includegraphics[width=0.04\textwidth]{{figures/mols/4}}} &  \multirow{2}{*}{\includegraphics[width=0.04\textwidth]{{figures/mols/5}}} &  \multirow{2}{*}{\includegraphics[width=0.04\textwidth]{{figures/mols/6}}} &  \multirow{2}{*}{\includegraphics[width=0.04\textwidth]{{figures/mols/7}}} &  \multirow{2}{*}{\includegraphics[width=0.04\textwidth]{{figures/mols/8}}} &  \multirow{2}{*}{\includegraphics[width=0.04\textwidth]{{figures/mols/9}}} &  \multirow{2}{*}{\includegraphics[width=0.04\textwidth]{{figures/mols/10}}} &  \multirow{2}{*}{\includegraphics[width=0.04\textwidth]{{figures/mols/11}}} &  \multirow{2}{*}{\includegraphics[width=0.04\textwidth]{{figures/mols/12}}} &  \multirow{2}{*}{\includegraphics[width=0.04\textwidth]{{figures/mols/13}}} &  \multirow{2}{*}{\includegraphics[width=0.04\textwidth]{{figures/mols/14}}} &  \multirow{2}{*}{\includegraphics[width=0.04\textwidth]{{figures/mols/15}}} &  \multirow{2}{*}{\includegraphics[width=0.04\textwidth]{{figures/mols/16}}} &  \multirow{2}{*}{\includegraphics[width=0.04\textwidth]{{figures/mols/17}}} &  \multirow{2}{*}{\includegraphics[width=0.04\textwidth]{{figures/mols/18}}} &  \multirow{2}{*}{\includegraphics[width=0.04\textwidth]{{figures/mols/20}}} &  \multirow{2}{*}{\includegraphics[width=0.04\textwidth]{{figures/mols/19}}} &  \multirow{2}{*}{\includegraphics[width=0.04\textwidth]{{figures/mols/21}}} &  \multirow{2}{*}{\includegraphics[width=0.04\textwidth]{{figures/mols/22}}} &  \multirow{2}{*}{\includegraphics[width=0.04\textwidth]{{figures/mols/23}}} &  \multirow{2}{*}{\includegraphics[width=0.04\textwidth]{{figures/mols/24}}} &  \multirow{2}{*}{\includegraphics[width=0.04\textwidth]{{figures/mols/25}}} &  \multirow{2}{*}{\includegraphics[width=0.04\textwidth]{{figures/mols/26}}} &  \multirow{2}{*}{\includegraphics[width=0.04\textwidth]{{figures/mols/27}}} &  \multirow{2}{*}{\includegraphics[width=0.04\textwidth]{{figures/mols/28}}} &  \multirow{2}{*}{\includegraphics[width=0.04\textwidth]{{figures/mols/29}}} \\
      &&&&&&&&&&&&&&&&&&&&&&&&&&&&&\\
\midrule
\textbf{Count}   & 10000 &  7548 &                                                       6198 &                                                       5629 &                                                       3904 &                                                       2204 &                                                       1799 &                                                       1772 &                                                       1348 &                                                        1330 &                                                        1071 &                                                         741 &                                                         573 &                                                         375 &                                                         208 &                                                         204 &                                                         176 &                                                         156 &                                                         113 &                                                         113 &                                                          90&                                                          80 &                                                          77 &                                                          66 &                                                          54 &                                                          45 &                                                          37 &                                                          32 &                                                          31 \\
\textbf{Accuracy} &                                                          1.0 &                                                        0.997 &                                                        0.999 &                                                        0.986 &                                                         0.99 &                                                        0.969 &                                                        0.999 &                                                          1.0 &                                                        0.963 &                                                         0.997 &                                                         0.997 &                                                         0.933 &                                                         0.999 &                                                         0.954 &                                                         0.999 &                                                         0.988 &                                                         0.983 &                                                         0.982 &                                                         0.996 &                                                         0.995 &                                                         0.993 &                                                         0.991 &                                                         0.998 &                                                         0.995 &                                                         0.996 &                                                         0.996 &                                                         0.996 &                                                         0.998 &                                                         0.998 \\
\bottomrule
\end{tabular}
}

\resizebox{\textwidth}{!}{
\begin{tabular}{l|HHHHHHHHHHHHHHHHHHHHHHHHHHHHllllllllllllll}
\multirow{1}{*}{\textbf{Fragment}}     &  \includegraphics[width=0.2\textwidth]{{figures/mols/0}} &  \includegraphics[width=0.2\textwidth]{{figures/mols/1}} &  \includegraphics[width=0.2\textwidth]{{figures/mols/2}} &  \includegraphics[width=0.2\textwidth]{{figures/mols/3}} &  \includegraphics[width=0.2\textwidth]{{figures/mols/4}} &  \includegraphics[width=0.2\textwidth]{{figures/mols/5}} &  \includegraphics[width=0.2\textwidth]{{figures/mols/6}} &  \includegraphics[width=0.2\textwidth]{{figures/mols/7}} &  \includegraphics[width=0.2\textwidth]{{figures/mols/8}} &  \includegraphics[width=0.2\textwidth]{{figures/mols/9}} &  \includegraphics[width=0.2\textwidth]{{figures/mols/10}} &  \includegraphics[width=0.2\textwidth]{{figures/mols/11}} &  \includegraphics[width=0.2\textwidth]{{figures/mols/12}} &  \includegraphics[width=0.2\textwidth]{{figures/mols/13}} &  \includegraphics[width=0.2\textwidth]{{figures/mols/14}} &  \includegraphics[width=0.2\textwidth]{{figures/mols/15}} &  \includegraphics[width=0.2\textwidth]{{figures/mols/16}} &  \includegraphics[width=0.2\textwidth]{{figures/mols/17}} &  \includegraphics[width=0.2\textwidth]{{figures/mols/18}} &  \includegraphics[width=0.2\textwidth]{{figures/mols/19}} &  \includegraphics[width=0.2\textwidth]{{figures/mols/20}} &  \includegraphics[width=0.2\textwidth]{{figures/mols/21}} &  \includegraphics[width=0.2\textwidth]{{figures/mols/22}} &  \includegraphics[width=0.2\textwidth]{{figures/mols/23}} &  \includegraphics[width=0.2\textwidth]{{figures/mols/24}} &  \includegraphics[width=0.2\textwidth]{{figures/mols/25}} &  \includegraphics[width=0.2\textwidth]{{figures/mols/26}} &  \includegraphics[width=0.2\textwidth]{{figures/mols/27}} &  \includegraphics[width=0.04\textwidth]{{figures/mols/28}} &  \includegraphics[width=0.04\textwidth]{{figures/mols/29}} &  \includegraphics[width=0.04\textwidth]{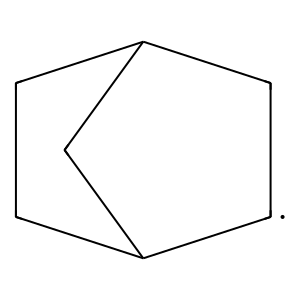} &  \includegraphics[width=0.04\textwidth]{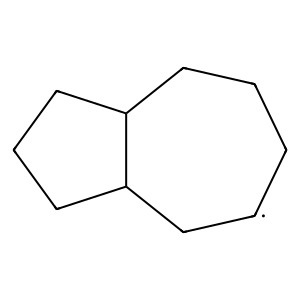} &  \includegraphics[width=0.04\textwidth]{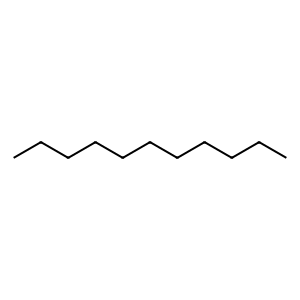} &  \includegraphics[width=0.04\textwidth]{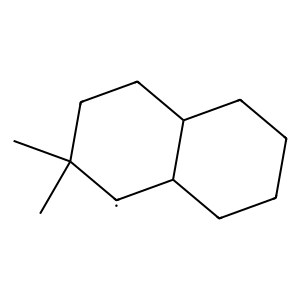} &  \includegraphics[width=0.04\textwidth]{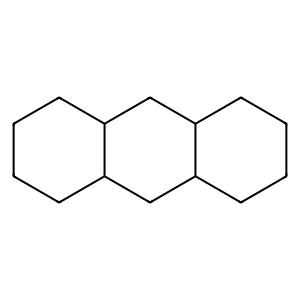} &  \includegraphics[width=0.04\textwidth]{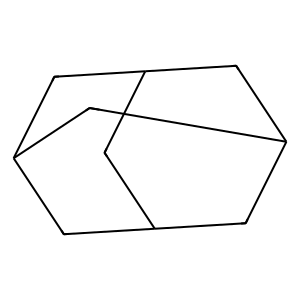} &  \includegraphics[width=0.04\textwidth]{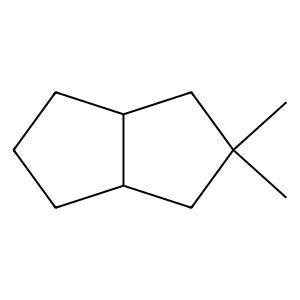} &  \includegraphics[width=0.04\textwidth]{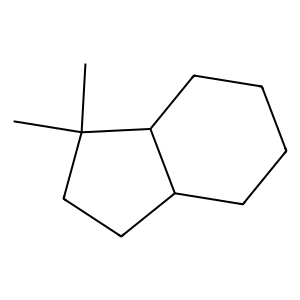} &  \includegraphics[width=0.04\textwidth]{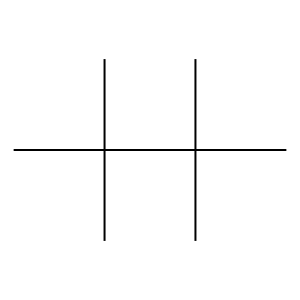} &  \includegraphics[width=0.04\textwidth]{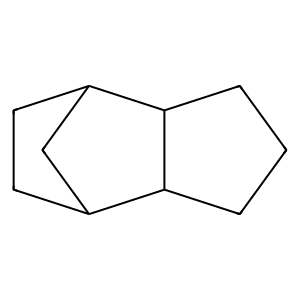} &  \includegraphics[width=0.04\textwidth]{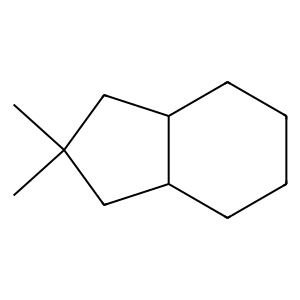} &  \includegraphics[width=0.04\textwidth]{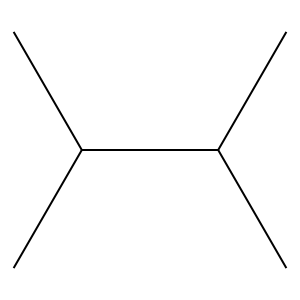} \\
\midrule
\textbf{Count}   &                                                      12862.0 &                                                      10000.0 &                                                       7548.0 &                                                       6198.0 &                                                       5629.0 &                                                       3904.0 &                                                       2204.0 &                                                       1799.0 &                                                       1772.0 &                                                       1348.0 &                                                        1330.0 &                                                        1071.0 &                                                         741.0 &                                                         573.0 &                                                         375.0 &                                                         208.0 &                                                         204.0 &                                                         176.0 &                                                         156.0 &                                                         113.0 &                                                         113.0 &                                                          90.0 &                                                          80.0 &                                                          77.0 &                                                          66.0 &                                                          54.0 &                                                          45.0 &                                                          37.0 &                                                          32 &                                                          31 &                                                          28 &                                                          25 &                                                          23 &                                                          19 &                                                          19 &                                                          18 &                                                          18 &                                                          17 &                                                          15 &                                                          15 &                                                          15 &                                                          13 \\
\textbf{Accuracy} &                                                          1.0 &                                                          1.0 &                                                        0.997 &                                                        0.999 &                                                        0.986 &                                                         0.99 &                                                        0.969 &                                                        0.999 &                                                          1.0 &                                                        0.963 &                                                         0.997 &                                                         0.997 &                                                         0.933 &                                                         0.999 &                                                         0.954 &                                                         0.999 &                                                         0.988 &                                                         0.983 &                                                         0.982 &                                                         0.995 &                                                         0.996 &                                                         0.993 &                                                         0.991 &                                                         0.998 &                                                         0.995 &                                                         0.996 &                                                         0.996 &                                                         0.996 &                                                         0.998 &                                                         0.998 &                                                         0.996 &                                                         0.997 &                                                           1.0 &                                                         0.998 &                                                         0.997 &                                                         0.998 &                                                           1.0 &                                                           1.0 &                                                         0.998 &                                                         0.998 &                                                           1.0 &                                                         0.999 \\
\bottomrule
\end{tabular}
}

\label{tab:more_counting}
\end{table}
We use the MagNet \cite{hetzel_magnet_2023} fragmentation to fragment all graphs in the ZINC-subset dataset. We sort the fragments by number of occurrences in the training set. For each of the 28 most common substructure we train our model to predict the counts of these substructures. As model parameters we use three layers of message passing with a hidden dimension of 120. For the final readout function we use a sum aggregation and a two layer MLP.
We train our model using the MAE loss for 200 epochs with a learning rate of 0.001 and a reduce-on-plateau learning rate scheduling. We report the accuracy (percentage of graphs where rounded prediction equals the ground-truth count) on the test set. \cref{tab:more_counting} shows the complete table of all substructures.

\subsection{Long-range Interaction: Recovery rate}
In our synthetic long-range experiment, we consider a graph consisting of two rings connected by a path \cref{fig:recovery}. One node in the graph is the designated source node (marked by a star). The feature of the source node is initialized with one-hot-encoding of one of 10 different classes. All other node features are initialized with a constant encoding. 
For every node $t$ in the graph, we train a separate model to predict the class of the source node $s$, i.e. the target node $t$ has to reconstruct a message from the source node. The number of layers of the models is $\max(d(s,t),3)$, ensuring that the target can receive messages from the source. We train the model with the cross entropy loss between the prediction at the target node and the true class of the source. We compare the results of models that have no fragmentation, a ring fragmentation and a ring-path fragmentation. We use a our model without batchnorm and a hidden dimension of 64. We train the model for a maximum of 200 epochs with a starting learn rate of 0.001 and average the results over at least five seeds.

\subsection{Generalization: Rarity}
For the experiment in \cref{tab:rarity}, we report the MAE of the ZINC-full validation set grouped by the frequency of the rarest fragment in the molecule. The frequency of a fragment is defined as the fraction of molecules that contain the fragment. As the fragmentation scheme, we use the simple Rings fragmentation.

\subsection{Generalization: QM9}
To perform our generalization experiment on QM9, we transform the edge and node features of the molecular graphs in QM9 so that they have the same node features and edge features as the graphs in the ZINC dataset. Additionally, we do not use any molecular graphs that contain atom types that do not appear in the ZINC dataset. We calculate penalized logP as ground truth.
Then, we trained our model and GRIT on ZINC-full and tested them on the transformed QM9 dataset.

\section{Downstream tasks using substructures}
Many other molecular tasks beyond property prediction can benefit from substructure information, highlighting the broader potential applications of our model.

\textbf{Motifs for Drug Discovery}
Motifs and specific substructures are important inductive biases in molecular generation, optimization, and scaffolding tasks \cite{hu2023deep, sommer_power_2023, du2022molgensurvey}. Employing a set of fragments can simplify the generation process and increase the chemical validity of the generated molecules. Given a fragmentation procedure, the fragments are aggregated into a vocabulary of motifs through complete enumeration of the dataset \cite{jin_junction_2019, jin2020hierarchical, geng2023novo}, top-k selection \cite{kong_molecule_2022, maziarz_learning_2022} or consolidation into Murcko scaffolds \cite{hetzel_magnet_2023}. Encoders for molecule generation often integrate motifs as node features or via additional higher-level encoder networks, as the decoder is explicitly tasked with reconstructing the set of motifs from a given embedding.

\textbf{Pretaining}
In the context of using GNNs for drug discovery, incorporating motifs as part of a pretraining phase has been shown to improve representation learning capabilities. \citet{zang2023hierarchical} integrates higher-level structures as nodes in a graph and leverages the graph's hierarchy for self-supervised pretraining. Similarly, \citet{zhang2021motif} propose a GNN that operates on a two-tiered graph and predicts the sequence of motifs during network pretraining. To improve the encoding of higher-level structures, \citet{inae2023motif} suggest a motif-aware pretraining technique, which masks entire motifs during the pretraining phase.

\end{document}